\documentclass{article}

\newif\ifpreprint
\preprinttrue
\usepackage[preprint]{neurips_2024}

\usepackage[utf8]{inputenc} % allow utf-8 input
\usepackage[T1]{fontenc}    % use 8-bit T1 fonts
\usepackage{hyperref}       % hyperlinks
\hypersetup{
    colorlinks=true,
    citecolor=blue,
    linkcolor=blue
}
\usepackage{url}            % simple URL typesetting
\usepackage{booktabs}       % professional-quality tables
\usepackage{amsfonts}       % blackboard math symbols
\usepackage{nicefrac}       % compact symbols for 1/2, etc.
\usepackage{microtype}      % microtypography
\usepackage{xcolor}         % colors
\usepackage{enumitem}
\usepackage{amsmath}
\usepackage{amsthm}
\newtheorem{theorem}{Theorem}

\newtheorem{lemma}{Lemma}

\newtheorem{assumption}{Assumption}
\usepackage{multirow}
\usepackage{mathtools}
\usepackage{wrapfig}
\usepackage{algorithm}
\usepackage[algo2e, vlined, ruled, linesnumbered]{algorithm2e}
\usepackage{bm}
\usepackage{prettyref}

\DeclareMathOperator{\E}{\mathbb{E}}

\DeclareMathOperator{\argmin}{\arg\min}

\DeclareMathOperator*{\argmax}{\arg\,\max}

\usepackage{color}
\usepackage[normalem]{ulem}
\newcommand{\commentout}[1]{}

\newcommand{\ip}[1]{\left\langle #1 \right\rangle}

\newcommand{\dts}[3]{D^{TS}_{#1}(#2,#3)}

\newcommand{\savehyperref}[2]{\texorpdfstring{\hyperref[#1]{#2}}{#2}}

\usepackage{tikz}
\usetikzlibrary{calc}

\newif\ifsup

\title{Incentive-compatible Bandits: Importance Weighting No More}

\author{%
  Julian Zimmert \\
  Google Research\\
  \texttt{zimmert@google.com} \\\\\\
  \And
  Teodor V. Marinov \\
  Google Research \\
  \texttt{tvmarinov@google.com} \\\\\\
}

\begin{document}

\maketitle

\begin{abstract}
We study the problem of incentive-compatible online learning with bandit feedback. In this class of problems, the experts are self-interested agents who might misrepresent their preferences with the goal of being selected most often. The goal is to devise algorithms which are simultaneously incentive-compatible, that is the experts are incentivised to report their true preferences, and have no regret with respect to the preferences of the best fixed expert in hindsight. \citet{freeman2020no} propose an algorithm in the full information setting with optimal $O(\sqrt{T \log(K)})$ regret and $O(T^{2/3}(K\log(K))^{1/3})$ regret in the bandit setting.

In this work we propose the first incentive-compatible algorithms that enjoy $O(\sqrt{KT})$ regret bounds. We further demonstrate how simple loss-biasing allows the algorithm proposed in \citet{freeman2020no} to enjoy $\tilde O(\sqrt{KT})$ regret. As a byproduct of our approach we obtain the first bandit algorithm with nearly optimal regret bounds in the adversarial setting which works entirely on the observed loss sequence without the need for importance-weighted estimators.
Finally, we provide an incentive-compatible algorithm that enjoys asymptotically optimal best-of-both-worlds regret guarantees, i.e., logarithmic regret in the stochastic regime as well as worst-case $O(\sqrt{KT})$ regret.
\end{abstract}

\section{Introduction}
Adversarial multi-armed bandits is a seminal online learning problem with applications in experimental design, online advertisement and more.
Bandits problem are characterized by limited feedback given to the learner in every round, the so-called bandit feedback.
In the adversarial setting, where the loss sequence of each expert is assumed bounded but otherwise can be arbitrarily generated, algorithms typically use importance weighted loss estimators. Such estimators can sometimes lead to very large or potentially unbounded losses which is often undesirable in practice. 
% To the best of our knowledge, all existing algorithms rely on constructing (almost) unbiased loss estimates to deal with bandit feedback.
% Importance weighting makes the loss-range large or potentially unbounded, requiring additional workarounds when adapting full-information online learning problems to the bandit setting.

One particular area where unbounded losses have hindered progress is incentive-compatible online learning with bandit feedback. 
% (see Section~\ref{sec: incentive}).
Incentive-compatible online learning is characterized by a set of self-interested experts who do not necessarily report their true belief of their best action but rather target being selected as often as possible by the algorithm. \citet{freeman2020no} show that classical algorithms like Hedge~\citep{littlestone1994weighted, vovk1995game, freund1997decision} may not be incentive-compatible. It turns out that a sufficient condition for an update rule on the distribution over experts to be incentive compatible is linearity of the update rule in the loss and boundedness of the loss~\citep{lambert2008self, lambert2015axiomatic}. A natural set of algorithms satisfying the above conditions in the full information setting is the Prod family in which the algorithms adjust the sampling probability of each expert by a multiplicative update of the loss
$\pi_{t+1,i} = \pi_{t,i}(1-\eta\ell_{t,i} +\lambda)$, where $\lambda$ is a normalization constant. The extension to the bandit setting proposed by \citet{freeman2020no} is to replace $\ell_{t}$ by a loss estimate $\hat \ell_t$.
This update requires careful control of the magnitude of $\hat\ell_t$ to ensure valid probability distributions.
In fact, this is the backbone of WSU-UX algorithm~\citep{freeman2020no}, which enjoys a regret bound of the order $O(T^{2/3})$.   
% It has been shown that no tuning of learning rate and mixed-in exploration is able to circumvent a $\Omega(T^\frac{2}{3})$ worst-case regret.
% It is an open problem if incentive-compatible bandits are fundamentally harder than regular bandits.
It is an open problem if incentive-compatible bandits are fundamentally harder than regular bandits or there exist algorithms with $O(\sqrt{T})$ regret guarantees.
The main contributions of our work are as follows.
\begin{enumerate}
    \item We propose a simple incentive-compatible algorithm with nearly optimal $\tilde O(\sqrt{KT})$ regret guarantee.
    As a side result, this algorithm is the first, to the best of our knowledge, which operates directly on the observed sequence of losses without the need of importance weighting in the adversarial setting.
    \item We show that a loss shift in the WSU-UX update improves the regret bound to $\tilde O(\sqrt{KT})$.
    \item We propose an incentive-compatible algorithm that adjusts to the hardness of the problem. In stochastic environments, we guarantee $\log(T)$ regret and in the adversarial setting we guarantee optimal $O(\sqrt{KT})$ regret.
\end{enumerate}

\section{Problem setting and related work}
\subsection{Adversarial bandits}
The adversarial bandit problem is defined as follows.
In every round $t=1,\dots, T$, an (oblivious) adversary selects a loss $\ell_t\in[0,1]^K$ (it is possible to extend the loss range to $[-1,1]^K$) unknown to the agent. The agent simultaneously selects an expert $A_t \in [K]$.
The agent incurs and observes the loss $\ell_{t,A_t}$, but does not see the losses of other experts.
The goal is to minimize the pseudo-regret\footnote{For the rest of the paper we refer to pseudo-regret as regret for simplicity.}
\begin{align*}
    \max_{i\in[K]}\E\left[\sum_{t=1}^T\ell_{t,A_t}-\ell_{t,i}\right]\,,
\end{align*}
where the expectation is taken with respect to all randomness of the algorithm, losses and experts.
Popular families of algorithms for this problem setting include online mirror descent (OMD) and follow the regularized leader (FTRL).
Typically the algorithms use unbiased loss estimates of the loss vector via importance weighting: $\hat\ell_{t,i}=\frac{\ell_{t,i}}{\pi_{t,i}}\mathbb{I}(A_t=i)$, where $\mathbb{I}(E)$ denotes the indicator function function for the event $E$. We note that other types of importance weighted estimators have been used in literature such as the implicit exploration estimator~\citet{kocak2014efficient}, which has improved variance properties.
% but is not unbiased.
The algorithms are defined by a twice-differentiable convex potential function $F : \mathbb{R}^K \to \mathbb{R}$ and a learning rate schedule $\eta_t$, the agent maintains a distribution via
\begin{align*}
    &\pi_{t+1} = \argmin_{\pi \in \Delta([K])}\ip{\pi,\eta_{t}\sum_{s=1}^t\hat\ell_s} - F(\pi)\tag{FTRL}\,,\\
    &\pi_{t+1} = \argmin_{\pi \in \Delta([K])}\ip{\pi,\eta_{t}\hat\ell_t} - D_F(\pi,\pi_t)\tag{OMD}\,,
\end{align*}
where $D_F(y,x)=F(y)-F(x)-(y-x)^\top\nabla F(x)$ is the Bregman divergence of $F$ and $\Delta([K])$ denotes the probability simplex over the set $[K]$ of experts.

Common potentials in the bandit literature are given in Table~\ref{tab:potentials} and we refer to their respective Bregman divergences as $D_{KL},D_{TS}$ and $D_{LB}$ respectively.
The negative entropy is the potential which defines Hedge and Exp-3 (and derivatives)~\citep{littlestone1994weighted, vovk1995game, freund1997decision, auer2002nonstochastic, kocak2014efficient}.
% \tm{Not sure how many papers we should cite here}.
The $1/2$-Tsallis Entropy is the key to achieving optimal best-of-both-worlds regret guarantees as was first demonstrated by \citet{zimmert2021tsallis}. The log-barrier potential was used by \citet{agarwal2017corralling} to first solve the corralling of bandits problem and has found many applications in model-selection problems~\citep{foster2020adapting}, regret bounds which depend on the properties of the loss sequence~\citep{wei2018more, lee2020closer, lee2020bias} and various other bandit problems.
\begin{table}
\begin{center}
    \begin{tabular}{c|c|c|c}
         &Negentropy/KL divergence & $1/2$-Tsallis Entropy & Logbarrier\\\hline
         $F(\pi)$& $\sum_{i=1}^K\pi_{t,i}\log(\pi_{t,i})$&$ -2\sum_{i=1}^K\sqrt{\pi_{t,i}}$&$-\sum_{i=1}^K\log(\pi_{t,i})$ 
    \end{tabular}
\caption{Common potential functions}
\label{tab:potentials}
\end{center}
\end{table}

\subsection{Incentive-compatible online learning}
\label{sec: incentive}
Incentive-compatible online learning was introduced by \citet{freeman2020no}.
In this setting, they assume that in every round, a proper loss function $\mathcal{L}_t$ is drawn from an unknown distribution $\nu_t$. There are $K$-experts who each have an internal belief about the distribution $\nu_t$ and provide the learner with a recommendation $r_{t}$.
The learner chooses an expert $A_t$, receives their recommendation $r_{t,A_t}$, and suffers the loss $\mathcal{L}_t(r_{t,A_t})$.
The learner aims to minimize their expected loss in comparison to the performance of the best expert.
Let $r^\star_{t,i} = \argmin_r\E_{t,i}[\mathcal{L}_t(r)]$ denote the $i$-th expert's best response, where the expectation is with respect to the belief of the $i$-th expert conditioned on all observations until round $t$. The regret of the learner is then
\begin{align*}
    \max_{i\in[K]}\E\left[\sum_{t=1}^T\mathcal{L}_t(r_{t,A_t})-\mathcal{L}_t(r^\star_{t,i})\right]\,.
\end{align*}
We note that the learner's loss depends on $r_t$, but the baseline depends on the best response $r^\star_t$. Without ensuring that these two quantities are related, one cannot hope to obtain small regret. Ideally, the learner wants to ensure truthfulness, i.e. $r_{t} = r^\star_{t}$.

The caveat is that the experts are self-interested with the goal of maximizing the probability of being selected in every round, that is 
$r_{t,i}=\argmax_r\E_{t,i}[\pi_{t+1,i}\,|\,r_{t,i}=r], \forall t\in[T]$. We assume that the experts are aware of the learner's strategy.
Algorithms ensuring that the experts are truthful are called incentive-compatible.
If the algorithm is incentive-compatible, then we can assign $\ell_{t,i}=\mathcal{L}(r_{t,i})=\mathcal{L}(r^\star_{t,i})$ and treat this as a standard online learning problem.

As an illustrative example, assume every day, we invite a weather forecaster out of a pool of $K$ experts to make a prediction $p$ of whether it will rain the next day. We only obtain a prediction for the single expert we invite at time $t$. Given the outcome $I^{rain}_t$, we suffer the loss $\mathcal{L}_t(p)=(I^{rain}_t-p)^2$. The expert minimizes this loss by truthfully reporting their internal believe of the likelihood of rain $\E_{t,i}[I^{rain}_t]$. However, the expert is aware of the selection rule and only cares about their prestige and the probability of them taking the stage again in the next round.

\citet{freeman2020no} show that typical OMD and FTRL algorithms are in fact not incentive compatible.
A simple strategy to ensure incentive-compatibility is to restrict the algorithm class to linear in the loss update rules
\begin{align*}
    &\pi_{t+1,i} = L_{t,i,A_t}(\ell_{t,A_t})\,,\tag{Linear update rule}
\end{align*}
where $L_{t,i,A_t}$ are affine functions with negative derivative. For linear update rules, maximizing $\pi_{t+1,i}$ is equivalent to minimizing the expected value of $\mathcal{L}_t(r_{t,i})$ and hence these algorithms are incentive-compatible.

In the full-information setting, such algorithms have been proposed as variants of Prod, where the update rule is given by
\begin{align*}
\label{eq:vanilla-prod}
    \pi_{t+1,i} = \pi_{t,i}(1-\eta(\ell_{t,i}-\lambda_t))\,,\qquad\lambda_t=\sum_{j=1}^K\pi_{t,j}\ell_{t,j}\,. \tag{Vanilla-Prod}
\end{align*}
In the bandit feedback setting, \citet{freeman2020no} modify the algorithm by  mixing the uniform distribution in every round before sampling to control the variance of the estimator:
\begin{align*}
\label{eq:wsu-ux}
    \tilde\pi_{t,i} &= \frac{\gamma}{K} +(1-\gamma)\pi_{t,i}, \qquad\qquad
    A_t \sim\tilde\pi_{t,i}, \qquad\quad\hat\ell_{t,i} = \mathbb{I}(A_t=i)\frac{\ell_{t,i}}{\tilde\pi_{t,i}}\\
    \pi_{t+1,i} &= \pi_{t,i}(1-\eta(\hat\ell_{t,i}-\lambda_t))\,,\qquad\lambda_t=\sum_{j=1}^K\pi_{t,j}\hat\ell_{t,j}\,,\tag{WSU-UX}
\end{align*}
where $\gamma$ is the mixture coefficient.
WSU-UX is a linear update rule, however, in the worst case the algorithm will suffer $\Omega(T^\frac{2}{3})$ regret no matter how $\gamma$ or $\eta$ are chosen, even if they are chosen adaptively with respect to the time horizon. 
In this work, we restrict ourselves to the study of algorithms with linear update rules.

\subsection{Best of both worlds}
Adversarial losses are a worst-case assumption which might be overly conservative in practice.
% Assuming that losses are adversarial is a worst-case assumption that might be overly conservative in practice.
In more favourable environments such as stochastic regimes, where $\ell_{t,i}$ are sampled from a fixed distribution with sufficient separation between the smallest and second smallest expected loss, there exist algorithms which enjoy $\log(T)$ regret bounds.  
% instead of $\sqrt{T}$ and 
It is desirable to have algorithms automatically adapt to the hardness of the data.
\citet{zimmert2021tsallis} have shown that simple FTRL with Tsallis entropy and a data-independent learning rate schedule satisfies this property.
We investigate whether the restricted algorithm class of linear update rules can also obtain best of both worlds.

\section{Modifying WSU-UX for nearly optimal regret guarantees}
\label{sec:kl-prod}
We begin by presenting a minimal modification of Algorithm~\ref{eq:wsu-ux} which is sufficient for a regret guarantee of the order $O(\sqrt{KT\log(K)})$. The modification comes in the form of biasing the loss estimator by subtracting a term from the losses. Similar techniques have been used by \citet{foster2020adapting} in addressing a model-selection problem. We will give more intuition on the role of this biasing later. The only change in the update is as follows
\begin{align*}
    \tilde \ell_{t,i} &= \ell_{t,i}\left(1 - \frac{\eta}{\tilde \pi_{t,i}}\right),\qquad
    \hat\ell_{t,i} = \mathbb{I}(A_t = i)\frac{\tilde\ell_{t,i}}{\tilde \pi_{t,i}}.
\end{align*}
We note that with an appropriate choice of $K\eta\leq\gamma$ we have $0 \leq \tilde \ell_{t,i} \leq \ell_{t,i} \leq 1$ and so all the conditions for the results in Section 4 of \citet{freeman2020no} are satisfied. For completeness we restate the results we use below. 
\begin{lemma}[Lemma 4.1~\citep{freeman2020no}]
If $\eta K/\gamma \leq \frac{1}{2}$, the WSU-UX weights $\pi_t$ and $\tilde \pi_t$ are valid probability distributions for all $t\in[T]$.
\end{lemma}
\begin{lemma}[Lemma 4.3~\citep{freeman2020no}]
\label{lem:43}
For WSU-UX, the probability vectors $\{\pi_t\}_{t\in [T]}$ and loss estimators $\hat \ell_t$ satisfy the following second order-bound
\begin{align*}
    \sum_{t=1}^T\sum_{i=1}^K \pi_{t,i}\hat\ell_{t,i} - \sum_{t=1}^T\hat \ell_{t,i^*} \leq \frac{\log(K)}{\eta} + \eta\sum_{t=1}^T \hat \ell_{t,i^*}^2 + \eta\sum_{t=1}^T\sum_{i=1}^K\pi_{t,i} \hat \ell_{t,i}^2.
\end{align*}
\end{lemma}
The above two results, together with the loss estimators based on $\tilde \ell_t, t\in[T]$ allow us to show the following regret bound.
\begin{theorem}
\label{thm:kl_prod}
Running WSU-UX with loss estimators based on $\tilde \ell_{t}, t\in[T]$ with $\gamma = \frac{\eta K}{2}, \eta = \Theta(\sqrt{\frac{\log(K)}{KT}})$ guarantees the following regret bound
\begin{align*}
    \sum_{t=1}^T \E[\ell_{t,A_t} - \ell_{t,i^*}] \leq O(\sqrt{KT\log(K)}). 
\end{align*}
\end{theorem}
\begin{proof}
WLOG we assume that $T\geq K$. We begin with the bound from Lemma~\ref{lem:43}. The second and third term in the RHS of the inequality are bounded as is standard in the Exp3 analysis
\begin{align*}
    \eta\sum_{t=1}^T\sum_{i=1}^K\pi_{t,i} \E[\hat \ell_{t,i}^2|\mathcal{F}_{t-1}] \leq \frac{\eta T K}{1-\gamma}\leq 2\eta T K, \qquad \eta\sum_{t=1}^T \E[\hat\ell_{t,i^*}^2] \leq \eta \sum_{t=1}^T \E\left[\frac{\tilde \ell_{t,i^*}^2}{\tilde\pi_{t,i^*}}\right] \leq \eta\sum_{t=1}^T \E\left[\frac{\ell_{t,i^*}^2}{\tilde \pi_{t,i^*}}\right],
\end{align*}
where $\mathcal{F}_{t-1}$ is the filtration generated by the random play and randomness of the losses up to time $t-1$.
We now consider the expectation of the LHS which evaluates to
\begin{align*}
    \sum_{t=1}^T\sum_{i=1}^K \E[\pi_{t,i}\hat\ell_{t,i}] - \sum_{t=1}^T\E[\hat \ell_{t,i^*}] &= \sum_{t=1}^T\sum_{i=1}^K\E[\pi_{t,i} \tilde \ell_{t,i}] - \sum_{t=1}^T\E[\tilde \ell_{t,i^*}] = \sum_{t=1}^T\sum_{i=1}^K\E[\pi_{t,i}\ell_{t,i}] - \sum_{t=1}^T\E[\ell_{t,i^*}]\\
    &+ \sum_{t=1}^T \E\left[\frac{\eta\ell_{t,i^*}}{\tilde \pi_{t,i}}\right] - \sum_{t=1}^T\sum_{i=1}^K \E\left[\frac{\eta \pi_{t,i}\ell_{t,i}}{\tilde \pi_{t,i}}\right]\geq \eta \sum_{t=1}^T \E\left[\frac{\ell_{t,i^*}^2}{\tilde \pi_{t,i}}\right] - 2\eta T K\\
    &+\sum_{t=1}^T\sum_{i=1}^K\E[\pi_{t,i}\ell_{t,i}] - \sum_{t=1}^T\E[\ell_{t,i^*}].
\end{align*}
Thus combining the bounds on the LHS and RHS we have
\begin{align*}
    \sum_{t=1}^T \sum_{i=1}^K \E[\pi_{t,i}\ell_{t,i}] - \sum_{t=1}^T \E[\ell_{t,i^*}] \leq \frac{\log(K)}{\eta} + 4\eta T K + \eta \sum_{t=1}^T \E\left[\frac{\ell_{t,i^*}^2}{\tilde \pi_{t,i}}\right] - \eta \sum_{t=1}^T \E\left[\frac{\ell_{t,i^*}^2}{\tilde \pi_{t,i}}\right].
\end{align*}
To complete the proof we only note that $\sum_{t=1}^T \sum_{i=1}^K \E[\pi_{t,i}\ell_{t,i}] - \sum_{t=1}^T \E[\ell_{t,A_t}] \leq 2T\gamma = \eta K T$. 
\end{proof}

\subsection{Intuition on biasing the update and the Prod family of algorithms}
The \ref{eq:vanilla-prod} update can be seen as a first order approximation to the Hedge update. Consider the Hedge update
\begin{align*}
    \pi_{t+1,i} = \pi_{t,i}\exp(-\eta(\hat\ell_{t,i}-\lambda_t))\,,
\end{align*}
where $\lambda_t$ is a normalization factor. The first order approximation is
\begin{align*}
    \pi_{t+1,i} \approx \pi_{t,i}(1-\eta(\hat\ell_{t,i}-\lambda_t))
\end{align*}
tuning $\lambda_t$ such that $\sum_{i=1}^K\pi_{t+1,i}=1$ recovers \ref{eq:vanilla-prod}. We can now reason why \citet{freeman2020no} are not able to obtain the min-max optimal regret bounds by tuning $\eta$ and $\gamma$ alone, which is that the first order approximation is loose. The loss-biasing introduced in the previous section acts as a correction term which approximately makes the \ref{eq:vanilla-prod} update equal the second order approximation of the Hedge update. Indeed, we have
\[
\pi_{t+1,i} = \pi_{t,i}\exp(-\eta(\hat\ell_{t,i}-\lambda_t))\approx \pi_{t,i}\left(1-\eta(\hat\ell_{t,i}-\lambda_t)+\frac{\eta^2}{2}(\hat\ell_{t,i}-\lambda_t)^2\right)\,.
\]
Since $-\eta\hat\ell_{t,i}+\frac{\eta^2}{2}\hat\ell_{t,i}^2=-\eta\hat\ell_{t,i}\left(1-\frac{\eta\ell_{t,i}}{2\pi_{t,i}}\right)$,
our loss adjustment is proportional to the second-order correction. We cannot exactly correct the second order difference, because the update rule must stay linear and not quadratic in the loss. We solve this by slightly overcorrecting, biasing by a larger amount than the second order adjustments implies as necessary. Fortunately, the regret analysis is not sensitive towards this as we have shown in Theorem~\ref{thm:kl_prod}.

\section{Importance weighting free adversarial MAB with LB-Prod}
\label{sec:lb_prod}
% We now present our main contribution, a simple incentive-compatible bandit algorithm that obtains rate-optimal regret.
As already discussed, for linear update rules such as \ref{eq:vanilla-prod}, the range of the losses must be bounded to ensure $\pi_{t+1,i}\in (0,1)$.
If the Prod update was to use the masked loss $\tilde\ell_{t,i}=\ell_{t,i}\mathbb{I}(A_t=i)$ instead of the importance weighted update, there would be no need for controlling the magnitude of the losses.
Maybe surprisingly to readers familiar with the bandit literature, this is possible.
LB-Prod (Algorithm~\ref{alg: lb-prod}) is a Prod version using the masked loss $\tilde\ell_t$.
The main deviation from \ref{eq:vanilla-prod} is introducing a non-symmetric normalization term $\lambda_{t,i}$: 
\[\lambda_{t,i}=\frac{\pi_{t,i}\sum_{j=1}^K\pi_{t,j}\tilde\ell_{t,j}}{\sum_{j=1}^K\pi_{t,j}^2}=\frac{\pi_{t,i}\pi_{t,A_t}}{\sum_{j=1}^K\pi_{t,j}^2}\ell_{t,A_t}\,.\]
Since  $\frac{\pi_{t,i}\pi_{t,A_t}}{\sum_{j=1}^K\pi_{t,j}^2}\leq \frac{\frac{1}{2}(\pi_{t,i}^2+\pi_{t,A_t}^2)}{\sum_{j=1}^K\pi_{t,j}^2}\leq 1$, the range of $\lambda_{t,i}$ is bounded by the range of $\ell_{t}$.
  
The following theorem shows that this simple algorithm is rate optimal under the right tuning.

\begin{algorithm2e}
\caption{LB-Prod}
\label{alg: lb-prod}
\textbf{Input:} Parameter $\eta<1$. \\
$\forall i\in[K]:\,\pi_{1,i}=\frac{1}{K}$.\\
\For{$t=1,\dots, T$}{
Play $A_t\sim \pi_{t}$ and observe $\ell_{t,A_t}$.\\
Construct $\forall i\in[K]:\,\tilde\ell_{t,i}=\ell_{t,i}\mathbb{I}(A_t=i)$.\\
Set $\forall i\in[K]:\lambda_{t,i}=\frac{\pi_{t,i}\sum_{j=1}^K\pi_{t,j}\tilde\ell_{t,j}}{\sum_{j=1}^K\pi_{t,j}^2}$.\\
Update $\pi_{t+1,i}=\pi_{t,i}(1-\eta(\tilde\ell_{t,i}-\lambda_{t,i}))$.
}
\end{algorithm2e}

\begin{theorem}
\label{thm: lb}
For any sequence of losses $\ell_t\in[-1,1]^K$ and any $\eta<1$, LB-Prod produces valid distributions $\pi_t\in\Delta([K])$ and its regret  is bounded by
\begin{align*}
\sum_{t=1}^T \E[\ell_{t,A_t} - \ell_{t,i^*}]\leq
2+\frac{K\log(T)}{\eta} + \frac{2\eta T}{1-\eta}\,. 
\end{align*}
\end{theorem}
% The proof is deferred to the end of the section.
Tuning $\eta = \sqrt{\frac{\log(T)K}{2T}}$ results in a regret bound of $O(\sqrt{KT\log(T)})$ for any $T>\frac{K\log(T)}{2}$.

\subsection{Analysis of LB-Prod}
The following technical lemma is proven in the appendix.
\begin{lemma}
\label{lem: expectation}
For any timestep $t$ and arm $i$, it holds 
\begin{align*}
    &\E_t[\tilde\ell_{t,i}-\lambda_{t,i}] = \pi_{t,i}\left(\ell_{t,i}-c_t\right)\,\\
    &\E_t[(\tilde\ell_{t,i}-\lambda_{t,i})^2]\leq 2\pi_{t,i}\,,
\end{align*}
where $c_t  \in [-1,1]$ is an arm independent constant.
\end{lemma}
The main lemma of the analysis bounds the per-step regret compared to a Bregman divergence potential. This lemma is closely related to bounding the stability
term in OMD/FTRL analysis.
\begin{lemma}
\label{lem: per step}
For any time $t\in[T]$ and any $u\in\Delta([K])$, it holds
\begin{align*}
 \ip{\pi_t-u,\ell_t}+\E_t\left[\eta^{-1}D_{LB}(u,\pi_{t+1})\right]-\eta^{-1}D_{LB}(u,\pi_{t}) \leq \frac{2\eta}{1-\eta}
\end{align*}
\end{lemma}
\begin{proof}
\begin{align*}
&\ip{\pi_t-u,\ell_t}+\E_t\left[\eta^{-1}D_{LB}(u,\pi_{t+1})\right]-\eta^{-1}D_{LB}(u,\pi_{t})\\
    &=\ip{\pi_t-u,\ell_t} +\E_t\left[\sum_{i=1}^K\frac{u_i-\pi_{t+1,i}}{\eta\pi_{t+1,i}}-\frac{u_i-\pi_{t,i}}{\eta\pi_{t,i}} +\frac{1}{\eta}\log\left(\frac{\pi_{t+1,i}}{\pi_{t,i}}\right) \right]\\
    &=\ip{\pi_t-u,\ell_t} +\E_t\left[\sum_{i=1}^K\frac{u_i\left(1-\frac{\pi_{t+1,i}}{\pi_{t,i}}\right)}{\eta\pi_{t+1,i}} +\frac{1}{\eta}\log\left(1-\eta(\tilde\ell_{t,i}-\lambda_{t,i})\right) \right]\\
    &\leq\ip{\pi_t-u,\ell_t} +\E_t\left[\sum_{i=1}^K\frac{u_i(\tilde\ell_{t,i}-\lambda_{t,i})}{\pi_{t,i}(1-\eta(\tilde\ell_{t,i}-\lambda_{t,i}))} -\tilde\ell_{t,i}+\lambda_{t,i}\right]\tag{$\log(1+x)\leq x$}\\
    &\leq\ip{\pi_t-u,\ell_t} +\sum_{i=1}^K\left(\frac{u_i}{\pi_{t,i}}-1\right)\E_t[\tilde\ell_{t,i}-\lambda_{t,i}]+\eta\sum_{i=1}^K\frac{u_i}{\pi_{t,i}}\E_t\left[\frac{(\tilde\ell_{t,i}-\lambda_{t,i})^2}{1-\eta} \right]\\
    &\leq \ip{\pi_t-u,\ell_t}+\sum_{i=1}^K(u_i-\pi_{t,i})(\ell_{t,i}-c_t)+\frac{2\eta}{1-\eta}\sum_{i=1}^Ku_i= \frac{2\eta}{1-\eta}\tag{Lemma~\ref{lem: expectation}}\,.
\end{align*}
\end{proof}
Finally, we can prove the main regret guarantee.
\begin{proof}{\bf of Theorem~\ref{thm: lb}}
To show that this algorithm outputs proper probability distributions, note that
\begin{align*}
\sum_{i=1}^K\pi_{t+1,i}=\left(\sum_{i=1}^K\pi_{t,i}\right)-\eta\pi_{t,A_t}\ell_{t,A_t}+\eta\sum_{j=1}^K\frac{\pi_{t,A_t}\pi_{tj}^2}{\sum_{k=1}^K\pi_{tk}^2}\ell_{t,A_t}=\sum_{i=1}^K\pi_{t,i}=\dots=\sum_{i=1}^K\pi_{1,i}=1\,.
\end{align*}
Additionally
\begin{align*}
&|\tilde\ell_{t,A_t}-\lambda_{t,A_t}|=\frac{\sum_{j\neq A_t}\pi_{tj}^2}{\sum_{j=1}^K\pi_{tj}^2}|\ell_{t,A_t}|\leq |\ell_{t,A_t}|\\
\forall i\neq A_t:\,&|\tilde\ell_{t,i}-\lambda_{t,i}|= \frac{\pi_{t,i}\pi_{t,A_t}}{\sum_{j=1}^K\pi_{tj}^2}|\ell_{t,A_t}|\leq  \frac{\pi_{t,i}^2+\pi_{t,A_t}^2}{2\sum_{j=1}^K\pi_{tj}^2}|\ell_{t,A_t}|\leq \frac{1}{2}|\ell_{t,A_t}|\,.
\end{align*}
Hence  for any $\eta<1$, the probability of any arm is strictly positive.

\textbf{Regret: }
For any comparator $u^\star$, we define $u = u^\star+\frac{1}{T}\left(\pi_{1}-u^\star\right)$, which satisfies $\sum_{t=1}^T\ip{u-u^\star,\ell_t}\leq 2$.
Using Lemma~\ref{lem: per step}, we obtain by the telescoping sum of Bregman terms
\begin{align*}
    \E\left[\sum_{t=1}^T\ip{\pi_t-u,\ell_t}\right]\leq \frac{4\eta T}{1-\eta} +\eta^{-1}\E\left[D_{LB}(u,\pi_{1})-D_{LB}(u,\pi_{T+1})\right]\leq \frac{4\eta T}{1-\eta} + \frac{K\log(T)}{\eta}\,.
\end{align*}
\end{proof}
\subsection{Intuition of LB-Prod}
As in the case of \ref{eq:wsu-ux}, LB-Prod is also the first order approximation of another algorithm, namely importance-weighted Log-barrier with OMD updates. The update rule of Logbarrier OMD is
\begin{align*}
    \pi_{t+1,i} = \frac{\pi_{t,i}}{1-\eta\pi_{t,i}(\hat\ell_{t,i}-\lambda_t)}\approx \pi_{t,i}(1-\eta\pi_{t,i}(\hat\ell_{t,i}-\lambda_t))\,.
\end{align*}
Tuning $\lambda_t$ to ensure $\sum_{i=1}^K\pi_{t+1,i}=1$ and taking $\pi_{t,i}$ into the bracket results in the update rule of LB-Prod.
The curvature of the Log-barrier regularization is what ensures that the importance weighted loss $\hat\ell_{t,i}$ is always multiplied with its probability $\pi_{t,i}$, allowing to run the algorithm on the masked non-weighted loss sequence directly.
Additionally, the second order error of the approximation is only of size $\eta$, which is why the algorithm does not need to add a correction term.

\section{Best of both worlds algorithms}
As an extension of the Prod family we propose a linearization of FTRL with $1/2$-Tsallis entropy regularization. The resulting algorithm, TS-Prod (Algorithm~\ref{alg: ts-prod}) enjoys
% similar guarantees to FTRL, that is it has
a min-max optimal regret bound of $O(\sqrt{KT})$ in the adversarial setting and an asymptotically optimal $O(\sum_{i\neq i^*}\frac{\log(T)}{\Delta_i})$ regret in the stochastic setting where each $\ell_{t,i}$ is sampled from some unknown distribution in $[0,1]$ and $\Delta_i=\max_{i^\star\in[k]}\E[\ell_{t,i}-\ell_{t,i^\star}]$ are the gaps to the optimal arm. Here we assume that there is a unique arm with $\Delta_{i*} = 0$. At the end of the section we propose a second algorithm, conceptually similar to TS-Prod. This algorithm enjoys the same regret guarantees as TS-Prod, however, we include it to demonstrate a different regret bounding technique which reduces to a regret bound on stabilized OMD~\citep{fang2022online} with the $1/2$-Tsallis entropy potential.
\subsection{TS-Prod}
By using the same linear approximation of the 1-step OMD update with $\frac{1}{2}$-Tsallis-entropy potential, we obtain
\begin{align*}
    \pi_{t+1,i} =\frac{\pi_{t,i}}{(1+\eta_t\sqrt{\pi_{t,i}}(\hat\ell_{t,i}-\lambda_t))^2}\approx \pi_{t,i}\left(1-\frac{2\eta_t}{\sqrt{\pi_{t,i}}}(\tilde\ell_{t,i}-\pi_{t,i}\lambda_t)\right)\,.
\end{align*}

We again have the issues of ensuring bounded loss ranges and correcting for second-order approximation errors.
Additionally, our update rule is derived from the OMD update, while we want to approximate FTRL.
To move from OMD to FTRL, one needs to additionally add a stabilization term (see \cite{fang2022online}) to the loss of $\left(\eta_{t}^{-1}-\eta_{t-1}^{-1}\right)\nabla F(\pi_t)$.
We ensure all of this with a single loss biasing.

\begin{algorithm2e}
\caption{TS-Prod}
\label{alg: ts-prod}
$\forall i\in[K]:\,\pi_{1,i}=\frac{1}{K}$.\\
\For{$t=1,\dots, T$}{
Play $A_t\sim \pi_{t}$ and observe $\ell_{t,A_t}$.\\
Construct $\forall i\in[K]:$
\begin{align}
&\tilde\ell_{t,i} = \left(\ell_{t,i}-\frac{\eta_t(C_t-\frac{13}{2}\pi_{t,i})}{\sqrt{\pi_{t,i}}}\right)\mathbb{I}\left(A_t=i\right)\,,\\
&\text{ where }C_t=\frac{13}{2}+\left(\frac{1}{\eta_t^2}-\frac{1}{\eta_t\eta_{t-1}}\right)\notag
\end{align}
Update
\begin{align}
    &\pi_{t+1,i} = \pi_{t,i}\left(1-\frac{2\eta_t}{\sqrt{\pi_{t,i}}}(\tilde\ell_{t,i}-\lambda_{t,i})\right)\,,\text{ where }\lambda_{t,i}=\frac{\pi_{t,i}\sum_{j=1}^K\sqrt{\pi_{t,j}}\tilde\ell_{t,j}}{\sum_{j=1}^K\pi_{t,j}^\frac{3}{2}}\label{eq: update TS}
\end{align}
}
\end{algorithm2e}
\begin{theorem}
\label{thm: ts}
The regret of Algorithm~\ref{alg: ts-prod} with $\eta_t=\frac{1}{\sqrt{K+26t}}$ is bounded by $O(\sqrt{KT}+K\log(T))$ in the adversarial setting and by
% \begin{align*}
    % \Reg = O\left(\sqrt{KT}+K\log(T)\right)
% \end{align*}
$O\left(\sum_{i\neq i^\star}\frac{\log(T)}{\Delta_i}\right)$ in the stochastic setting.
% \begin{align*}
    % \Reg = O\left(\sum_{i\neq i^\star}\frac{\log(T)}{\Delta_i}\right).
% \end{align*}
\end{theorem}
\subsection{Analysis of TS-Prod}
\label{sec:ts-prod}
We first show that the loss biasing is sufficient to ensure that the distribution is well defined. 
\begin{lemma}
\label{lem: ts lower prob}
If $C_t$ is a non-increasing sequence, $\eta_t <\frac{2}{\sqrt{KC_t^2}}$ and $\eta_{t+1}^2 \leq \eta_t^2(1-13\eta_t^2)$ for all $t$, then the update rule of TS-Prod~\ref{alg: ts-prod} is well defined and satisfies $\pi_{t,i}> C_t^2\eta_t^2$ for any arm and loss sequence at all time steps.
\end{lemma}
Further, we show the following bound on the stability-like term. The proof is similar to that of Lemma~\ref{lem: per step}, however, is slightly more technically involved and we defer it to the appendix.
\begin{lemma}
\label{lem: per step TS}
For any time $t$ such that $\pi_{t,i}>C_t^2\eta_t^2$, it holds
\begin{align*}
 \ip{\pi_t-u,\ell_t}+\E_t\left[\frac{D_{TS}(u,\pi_{t+1})}{\eta_t}\right]-\frac{D_{TS}(u,\pi_{t})}{\eta_t} \leq \sum_{i=1}^K\left(2\eta_t\sqrt{\pi_{t,i}}(1-\pi_{t,i})-\left(\frac{1}{\eta_t}-\frac{1}{\eta_{t-1}}\right)\frac{u_i - \pi_{t,i}}{\sqrt{\pi_{t,i}}}\right)\,.
\end{align*}
\end{lemma}
We are ready to prove the main regret guarantee.
\begin{proof}{\bf of Theorem~\ref{thm: ts}}
We first show that the requirements of Lemma~\ref{lem: ts lower prob} are satisfied.
With $\eta_t = \frac{1}{\sqrt{K + 26t}}$, we have
$C_t = \frac{13}{2}+\left(K+26t-\sqrt{(K+26t)(K+26t-26)}\right)>2\,,$
which is monotonically decreasing.
$C_t>2$ ensures that $\eta_t=\frac{1}{\sqrt{K+26t}}<\frac{2}{\sqrt{KC_t^2}}$. Further we have
\begin{align*}
\frac{\eta_{t+1}^2}{\eta_t^2}=\frac{K+26t}{K+26(t+1)}= 1-\frac{26}{K+26(t+1)}\leq 1-\frac{4}{K+26t}=1-4\eta_t^2\,.
\end{align*}
With Lemma~\ref{lem: ts lower prob}, we can use Lemma~\ref{lem: per step TS} at all time-steps.
\begin{align*}
    \E\left[\sum_{t=1}^T\ip{\pi_t-u,\ell_t}\right]&\leq \E\Bigg[\sum_{t=1}^T\frac{D_{TS}(u,\pi_t)-D_{TS}(u,\pi_{t+1})}{\eta_t}\\
    &+\sum_{i=1}^K\left(2\eta_t\sqrt{\pi_{t,i}}(1-\pi_{t,i})-\left(\frac{1}{\eta_t}-\frac{1}{\eta_{t-1}}\right)\frac{\pi_{t,i}-u_i}{\sqrt{\pi_{t,i}}}\right)\Bigg]\\
    &\leq \E\Bigg[\sqrt{K + 26}D_{TS}(u,\pi_1) + \sum_{t=2}^T \left(\frac{1}{\eta_t} - \frac{1}{\eta_{t-1}}\right)D_{TS}(u,\pi_t)\\
    &+\sum_{t=1}^T\sum_{i=1}^K\left(2\eta_t\sqrt{\pi_{t,i}}(1-\pi_{t,i})-\left(\frac{1}{\eta_t}-\frac{1}{\eta_{t-1}}\right)\frac{\pi_{t,i}-u_i}{\sqrt{\pi_{t,i}}}\right)\Bigg]\\
    &\leq (K+9) + \sum_{t=2}^T \E\left[\left(\frac{1}{\eta_t}-\frac{1}{\eta_{t-1}}\right)\left(\sum_{i=1}^K 2(\sqrt{\pi_{t,i}} - \sqrt{u_i}) + \frac{u_i - \pi_{t,i}}{\sqrt{\pi_{t,i}}}\right)\right]\\
    &+\sum_{t=1}^T\sum_{i=1}^K\E\left[\left(2\eta_t\sqrt{\pi_{t,i}}(1-\pi_{t,i})-\left(\frac{1}{\eta_t}-\frac{1}{\eta_{t-1}}\right)\frac{u_i - \pi_{t,i}}{\sqrt{\pi_{t,i}}}\right)\right]\\
    &\leq 2K+26 + \sum_{t=1}^T\sum_{i=1}^K \E[2\eta_t\sqrt{\pi_{t,i}}(1-\pi_{t,i})]\\
    &+ \sum_{t=2}^T\sum_{i=1}^K\E\left[\left(\frac{1}{\eta_t}-\frac{1}{\eta_{t-1}}\right)\sum_{i=1}^K 2(\sqrt{\pi_{t,i}} - \sqrt{u_i})\right]\\
    &\leq 2K+26+\sum_{t=1}^T\left(\sum_{i\neq i^*}2\eta_t\E[\sqrt{\pi_{t,i}}] + \E\left[\sqrt{\pi_{t,i^*}}2\eta_t\sum_{i\neq i^*}\pi_{t,i}\right]\right)\\
    &+\sum_{t=2}^T\left(\frac{1}{\eta_t} - \frac{1}{\eta_{t-1}}\right)\sum_{i\neq i^*}2\E[\sqrt{\pi_{t,i}}]\\
    &\leq 2K + 26\sum_{t=1}^T \frac{4}{\sqrt{K+t}}\sum_{i\neq i^*}\left(\E[\sqrt{\pi_{t,i}}] + \frac{1}{2}\E[\pi_{t,i}]\right).
\end{align*}
% \jz{Before continuing here, make sure that the constants are right. $C_t$ must have a slightly larger constant than 2 to work.}
% \tm{Second sum is handled with self-bounding, what do we do about the first sum?}
% \jz{
% If we want to take the easiest path possible without optimizing for constants:
% \begin{align*}
%     \sum_{i}\sqrt{x_i}(1-x_i)=\sum_{i\neq i^\star}\left(\sqrt{x_i}(1-x_i)+\sqrt{x_{i^\star}}x_i\right)\leq 2\sum_{i\neq i^\star}\sqrt{x_i}\,.
% \end{align*}

% }
The remainder of the proof follows standard arguments using the self-bounding trick as in \citet{zimmert2021tsallis}. For details on the self-bounding trick see Section~\ref{sec:self-bounding}.
\end{proof}

\subsection{Reduction to $1/2$-Tsallis OMD}
In this section we discuss a second approach to deriving Prod algorithms with provable regret guarantees, which uses the fact that the Prod update is the linear approximation for the respective OMD update. We focus on the $1/2$-Tsallis entropy regularizer. 
% As already discussed, the Prod updates can be seen as a linearization of the respective mirror descent updates. In this section we present a second approach to deriving algorithms with provable regret guarantees, specifically for the $1/2$-Tsallis entropy regularizer, which uses the connection of the Prod update to the mirror descent update directly. 
In particular, we show an instantiation of the $1/2$-TS Prod update which is equivalent to a perturbed mirror descent update. This allows for a regret analysis which follows the standard stability and penalty decomposition for with the following challenge. We need to establish the instance dependent optimality of OMD in the stochastic setting while ensuring that the adversarial setting guarantees are preserved. The stochastic setting suggests a decreasing step-size as $\eta_t = \Theta(1/\sqrt{t})$, however, the vanilla OMD update can suffer linear regret with such a step-size schedule and thus a more careful approach needs to be taken.
% The regret analysis now follows the well established stability and penalty decomposition with the following caveats. First, we need to establish the optimality of mirror descent with $1/2$-Tsallis entropy potential with a decreasing step-size in the stochastic setting. The decreasing step-size is required for optimal dependence on the gap parameters of the stochastic losses. Next, we need to establish the optimality of mirror descent in the adversarial setting with the decreasing step-size schedule. We note that the vanilla OMD update with decreasing step-size as $\eta_t = \Theta(1/\sqrt{t})$ can suffer linear regret in the adversarial setting so a more careful approach needs to be taken. 
% In summary the analysis in this section will establish the optimality of OMD in both the stochastic and adversarial settings, which to the best of our knowledge is novel.
In summary we establish the optimality of OMD in both the stochastic and adversarial settings, which to the best of our knowledge is novel.

Our starting point is the \emph{stabilized} OMD update proposed by \cite{fang2022online} to establish regret bounds with decreasing step size.
% To derive OMD updates with a decreasing step size \cite{fang2022online} propose to stabilize the OMD update.
Stabilization is the process of mixing the gradient mapping, $\nabla F(\pi_t)$, of the current iterate with the gradient mapping of the first iterate, $\nabla F(\pi_1)$, in the mirror descent update in the dual space.
% The stabilization is carried out either in the primal space where at every step the resulting iterate is mixed with the first iterate (primal stabilization) or in the dual space where the gradient mapping, $\nabla F(\pi_t)$, of the current iterate is mixed with the gradient mapping of the first iterate $\nabla F(\pi_1)$.
% The dual stabilization in particular 
This mixing turns out to be equivalent to the negative biasing of losses introduced in TS-Prod and the modification of \ref{eq:wsu-ux} in Section~\ref{sec:kl-prod}. We now define the OMD update which is equivalent to a simpler version of the TS-Prod update (Algorithm~\ref{alg: ts-prod}). Let
% We now show a modification of TS-Prod so that the generated trajectory $\{\pi_t\}_{t=1}^T$ can be equivalently generated by a version of OMD with $1/2$-Tsallis potential which is dual stabalized. Further, we show that the mirror descent version enjoys the same regret guarantees as TS-Prod. As a side result, this implies that the dual stabilized OMD version of \cite{fang2022online} with $1/2$-Tsallis potential enjoys a best-of-both worlds guarantee. Let
\begin{equation}
\label{eq:tsmd_perturbed}
    \begin{aligned}
    \hat\ell_{t,i} &= \frac{\ell_{t,i}\mathbf{1}(I_t=i)}{\pi_{t,i} + \gamma_t} - \frac{1 - \xi_t}{\eta_{t+1}\sqrt{\pi_{t,i}}},\,\,
    \hat L_{t,i} = \hat \ell_{t,i} - \sum_{i'}\frac{\sqrt{\pi_{t,i'}}\hat\ell_{t,i'}}{\sum_{j}\pi_{t,j}^{3/2}}\\
    \pi_{t+1,i} &= \frac{\pi_{t,i}}{(1 + \eta_{t+1}\sqrt{\pi_{t,i}}(\hat L_{t,i} + \epsilon_{t,i}))^2},
    \end{aligned}
\end{equation}
where $\gamma_t = \Theta(1/t)$ is a term which keeps the variance and magnitude of $\hat\ell_{t,i}$ bounded and $\epsilon_{t,i}$ are perturbations which ensure
% \begin{align*}
    $\frac{1}{(1 + \eta_t\sqrt{\pi_{t,i}}(\hat L_{t,i} + \epsilon_{t,i}))^2} = 1 - 2\eta_t\sqrt{\pi_{t,i}}\hat L_{t,i}$, 
that is the OMD update in Equation~\ref{eq:tsmd_perturbed} is equivalent to the simplified TS-Prod update $\pi_{t+1,i} = \pi_{t,i}(1-2\eta_t\sqrt{\pi_{t,i}}\hat L_{t,i})$. 
% \end{align*}
The perturbations, $\epsilon_{t,i}$, are well controlled as we show in the following lemma.
\begin{lemma}
\label{lem:perturbation_bound}
For every $t\in[T],i\in[K]$, there exists $\epsilon_{t,i}$ such that
% \begin{align*}
    $\frac{1}{(1 + \eta_t\sqrt{\pi_{t,i}}(\hat L_{t,i} + \epsilon_{t,i}))^2} = 1 - 2\eta_t\sqrt{\pi_{t,i}}\hat L_{t,i},$
% \end{align*}
with $|\epsilon_{t,i}| = O(|\eta_t\sqrt{\pi_{t,i}}\hat L_{t,i}|)$.
\end{lemma}
Lemma~\ref{lem:perturbation_bound} allows us to further bound the first and second moments of $\epsilon_{t,i}$ and proceed with the analysis for the stochastic and adversarial cases by using the standard regret decomposition into a \emph{penalty} and \emph{stability} terms. In the stochastic case we can bound the two terms in the following way
\begin{lemma}
\label{lem:omd_penalty}
For stochastic losses the penalty term is bounded in expectation by
\begin{align*}
    % \E[\dts{t+1}{u}{\pi_{t+1}} - \dts{t}{u}{\pi_{t+1}}] &\leq 
    O\Bigg(\frac{\E\left[\left(\sum_{i\neq i^*}\pi_{t+1,i}\right)^2\right]\sqrt{K}\log(t)}{\sqrt{t}} + \frac{1}{\sqrt{t}}\land\frac{\E\left[\left(\sum_{i\neq i^*}\pi_{t+1,i}\right)^2\right]\log(KT)}{\sqrt{t}}\Bigg).
\end{align*}
% where $\dts{t}{u}{v} = \frac{1}{\eta_t}D_{TS}(u,v)$ and $\eta_t = \frac{1}{\sqrt{t}}, \gamma_{t} = O(\frac{\sqrt{K}}{t})$.
\end{lemma}
\begin{lemma}
\label{lem:omd_stability}
For stochastic losses the stability term is bounded by
\begin{align*}
    O\left(\frac{1}{\sqrt{t}}\sum_{i=1}^K \sqrt{\pi_{t,i}}(1-\pi_{t,i}) + \frac{K\sqrt{\pi_{t,i}}}{t^2} + \frac{K}{t}\right).
\end{align*}
\end{lemma}
The stochastic regret bound proof can now be completed by a careful self-bounding argument.

In the adversarial case we reduce the regret bound to that of \citet{fang2022online} in the following way. Let $\Phi = F + I_{\Delta^{K-1}}$ be potential defined by mixing the $1/2$-Tsallis potential together with the indicator function for the probability simplex. The update of Algorithm 2 (Dual Stabilized OMD) can then be written as
\begin{align*}
    \hat w_{t+1} &= \nabla \Phi(\pi_t) - \eta_t (\tilde \ell_t + \epsilon_t),\\
    \hat y_{t+1} &= \xi_t \hat w_{t+1} + (1-\xi_t) \nabla\Phi(\pi_1),\\
    \pi_{t+1} &= \nabla \Phi^*(\hat y_{t+1}),
\end{align*}
where $\xi_t = \frac{\eta_{t+1}}{\eta_t}$. It turns out that this update is equivalent to the OMD update with respect to $\hat L_{t,i}$ in Equation~\ref{eq:tsmd_perturbed}. This allows us to use the regret bound in Theorem 3~\citep{fang2022online}. Overall the regret of the perturbed OMD version is bounded as follows.
\begin{theorem}
\label{thm:tsmd_perturbed}
The regret of the algorithm defined by the update in Equation~\ref{eq:tsmd_perturbed} is bounded by
\begin{align*}
    O\left(\sum_{i\neq i^*} \frac{\log(T)}{\Delta_i} + \frac{K\log^2(1/\Delta_{min})}{\Delta_{min}} + K^{3/2}\right)
\end{align*}
in the stochastic case, where $\Delta_{min}$ is the smallest gap between the expected losses. Further the regret in the adversarial setting is bounded by $O(\sqrt{KT})$.
\end{theorem}

\section{Discussion}
We have provided an extensive study of incentive-compatible bandits.
We have negatively resolved an open question of whether incentive-compatibility as defined in \citet{freeman2020no} is  harder than regular bandits.
Using linear approximations, partly with second order corrections, allows to recover results from well studied algorithms in the literature. We even obtain an algorithm with best-of-both-world guarantees.
Our algorithms are conceptually simpler than existing bandit algorithms, they update the probability distributions with basic arithmetic operations without the need to solve optimization problems.
 
Our successes make it likely that one can transfer even more sophisticated methods, such as first-order, second-order, path-norm bounds and online learning with graph feedback to this framework. We leave this investigation to future work. 
Another open question is whether one can also satisfy other notions of incentive-compatibility, such as when the experts are not optimizing greedily for the next-step probability but consider a larger horizon.

\bibliographystyle{plainnat}
\bibliography{mybib}
%%%%%%%%%%%%%%%%%%%%%%%%%%%%%%%%%%%%%%%%%%%%%%%%%%%%%%%%%%%%

\appendix

%%%%%%%%%%%%%%%%%%%%%%%%%%%%%%%%%%%%%%%%%%%%%%%%%%%%%%%%%%%%

\newpage

% \appendix
\ifpreprint
\tableofcontents
\clearpage

\section{Technical Lemmas}
\begin{lemma}
\label{lem: technical 1}
\begin{align*}
    \min_{x\in[0,1]}f(x) = \min_{x\in[0,1]} \frac{x^3}{1-x}+\sqrt{1-x} \geq \sqrt{\frac{8}{9}}\,.
\end{align*}
\end{lemma}
\begin{proof}
We first show that the optimal point is smaller than $\frac{1}{9}$, by looking at the derivative
\begin{align*}
    f'(x) = \frac{3\sqrt{x}-x^\frac{3}{2}-(1-x)^\frac{3}{2}}{2(1-x)^2}\,.
\end{align*}
For the enumerator, we have for all $x\geq \frac{1}{9}$:
\begin{align*}
    3\sqrt{x}-x^\frac{3}{2}-(1-x)^\frac{3}{2}&>
    3\sqrt{x}-\max\{\sqrt{x},\sqrt{1-x}\}
    &\geq\min\{2\sqrt{x}, 3\sqrt{x}-1\}\geq 0
\end{align*}
Hence
\begin{align*}
    \min_{x\in[0,1]}f(x)=\min_{x\in[0,1/9]}f(x)>\min_{x\in[0,1/9]}\sqrt{1-x}=\sqrt{\frac{8}{9}}\,.
\end{align*}
\end{proof}
\begin{lemma}
\label{lem: technical 2}
For any $a,b\geq 0$ such that $a+b\geq 1$, it holds
\begin{align*}
    \frac{a}{\sqrt{b}}+\frac{b}{\sqrt{a}} \geq \sqrt{2}\,.
\end{align*}
\end{lemma}
\begin{proof}
We can assume w.l.o.g. that $a=1-b$, otherwise scale both $a$ and $b$ down and reduce the objective.
The resulting problem is symmetric with $a=\frac{1}{2}$ as the unique minimizer resulting in the statement.
\end{proof}
\section{Missing Proofs Section~\ref{sec:lb_prod}}
\begin{proof}[Proof of Lemma~\ref{lem: expectation}]
The expectation of $\tilde\ell_{ti}=\ell_{ti}\mathbb{I}(A_t=i)$ is obviously $\pi_{ti}\ell_{ti}$, hence by the definition of $\lambda_{ti}$, we have
\begin{align*}
    \E_t[\tilde\ell_{ti}-\lambda_{ti}]=\pi_{ti}\left(\ell_{ti}-\frac{\sum_{j=1}^K\pi_{tj}^2\ell_{tj}}{\sum_{j=1}^K\pi_{tj}^2}\right)\,.
\end{align*}
For the second part, we have
\begin{align*}
    \E_t[(\tilde\ell_{ti}-\lambda_{ti})^2]\leq \E_t[\tilde\ell_{ti}^2]+\E_t[\lambda_{ti}^2] &= \pi_{ti}\left(\ell_{ti}^2 + \pi_{ti}^2\frac{\sum_{j=1}^K\pi_{tj}^3\ell_{tj}^2}{\left(\sum_{k=1}^K\pi_{tk}^2\right)^2}\right)
    &\leq \pi_{ti}\left(1+\frac{\pi_{ti}\sum_{j=1}^K\pi_{tj}^3}{\left(\sum_{k=1}^K\pi_{tk}^2\right)^2}\right)\,.
\end{align*}
The proof is completed by noting
\begin{align*}
    \frac{\pi_{ti}\sum_{j=1}^K\pi_{tj}^3}{\left(\sum_{k=1}^K\pi_{tk}^2\right)^2}\leq \frac{\pi_{ti}\sum_{j=1}^K\pi_{tj}^3}{\left(\sum_{k=1}^K\pi_{tk}^3\right)^\frac{4}{3}}=\frac{\pi_{ti}}{\left(\sum_{k=1}^K\pi_{tk}^3\right)^\frac{1}{3}}\leq 1\,.
\end{align*}
\end{proof}

\section{Missing Proofs Section~\ref{sec:ts-prod}}
\subsection{Minimal probability}
\begin{lemma}
\label{lem:ts_moments}
Assume that $\pi_{ti}>C_t^2\eta_t^2$ holds for all arms, then
\begin{align*}
    &\E_t\left[\tilde\ell_{ti}-\lambda_{ti}\right]=\pi_{ti}(\ell_{ti}-c_t)-\eta_t \sqrt{\pi_{ti}}(C_t-\frac{13}{2}\pi_{ti})\\
    &\E_t\left[(\tilde\ell_{ti}-\lambda_{ti})^2\right]\leq \frac{13}{8}\pi_{ti}(1-\pi_{ti})
\end{align*}
\end{lemma}
\begin{proof}
$\lambda_{ti}=\pi_{ti}\lambda_t$ for $\lambda_t=\frac{\sqrt{\pi_{t,A_t}}\tilde\ell_{t,A_t}}{\sum_{j=1}^K\pi_{tj}^\frac{3}{2}}$.
\begin{align*}
    \E_t\left[\tilde\ell_{ti}-\lambda_{ti}\right] &= \pi_{ti}\ell_{ti}-\eta_t\sqrt{\pi_{ti}}(C_t-\frac{13}{2}\pi_{ti})-\pi_{ti}\E_t\left[\lambda_t\right]\,.
\end{align*}
For the second claim, note that $\tilde\ell_{ti}\in[-1,1]$ by the condition on $\pi_{ti}$. Hence
\begin{align*}
    \E_t\left[(\tilde\ell_{ti}-\lambda_{ti})^2\right]&\leq \pi_{ti}\left(1-\frac{\pi_{ti}^\frac{3}{2}}{\sum_{j=1}^K\pi_{tj}^\frac{3}{2}}\right)^2 + \sum_{j\neq i}\pi_{tj}\left(\frac{\pi_{ti}\sqrt{\pi_{tj}}}{\sum_{j=1}^K\pi_{tj}^\frac{3}{2}}\right)^2\\
    &=\pi_{ti}\left(\frac{(\sum_{j\neq i}\pi_{tj}^\frac{3}{2})^2+\pi_{ti}\sum_{j\neq i}\pi_{tj}^2}{(\sum_{j=1}^K\pi_{tj}^\frac{3}{2})^2}\right)\\
    &=\pi_{ti}(1-\pi_{ti})\left(\frac{(1-\pi_{ti})^2(\sum_{j\neq i}\tilde\pi_{tj}^\frac{3}{2})^2+\pi_{ti}(1-\pi_{ti})\sum_{j\neq i}\tilde\pi_{tj}^2}{(\pi_{ti}^\frac{3}{2}+(1-\pi_{ti})^\frac{3}{2}\sum_{j\neq i}\tilde\pi_{tj}^\frac{3}{2})^2}\right)\,,
\end{align*}
where $\tilde\pi_{tj}=\frac{\pi_{tj}}{1-\pi_{ti}}$.
We bound the two terms in the bracket separately, for the first term we have
\begin{align*}
\left(\frac{(1-\pi_{ti})\sum_{j\neq i}\tilde\pi_{tj}^\frac{3}{2}}{\pi_{ti}^\frac{3}{2}+(1-\pi_{ti})^\frac{3}{2}\sum_{j\neq i}\tilde\pi_{tj}^\frac{3}{2}}\right)^2&\leq
\left(\frac{(1-\pi_{ti})}{\pi_{ti}^\frac{3}{2}+(1-\pi_{ti})^\frac{3}{2}}\right)^2\tag{$\sum_{j\neq i}\tilde\pi_{tj}=1$}\\
&\leq \frac{9}{8} \tag{Lemma~\ref{lem: technical 1}}
\end{align*}
The second term is
\begin{align*}
    \frac{\pi_{ti}(1-\pi_{ti})\sum_{j\neq i}\tilde\pi_{tj}^2}{(\pi_{ti}^\frac{3}{2}+(1-\pi_{ti})^\frac{3}{2}\sum_{j\neq i}\tilde\pi_{tj}^\frac{3}{2})^2}&\leq \frac{\pi_{ti}(1-\pi_{ti})(\sum_{j\neq i}\tilde\pi_{tj}^\frac{3}{2})^\frac{4}{3}}{(\pi_{ti}^\frac{3}{2}+(1-\pi_{ti})^\frac{3}{2}\sum_{j\neq i}\tilde\pi_{tj}^\frac{3}{2})^2} \\
    &= \left(\frac{\pi_{ti}}{\sqrt{1-\pi_{ti}}(\sum_{j\neq i}\tilde\pi_{tj}^\frac{3}{2})^\frac{2}{3}} +\frac{(1-\pi_{ti})(\sum_{j\neq i}\tilde\pi_{tj}^\frac{3}{2})^\frac{1}{3}}{\sqrt{\pi_{ti}}}\right)^{-2}\\
    &\leq \frac{1}{2}\tag{Lemma~\ref{lem: technical 2}}
\end{align*}
\end{proof}

\begin{lemma}[Lemma~\ref{lem: ts lower prob}]
% \label{lem: ts lower prob}
If $C_t$ is a non-increasing sequence, $\eta_t <\frac{2}{\sqrt{KC_t^2}}$ and $\eta_{t+1}^2 \leq \eta_t^2(1-13\eta_t^2)$ for all $t$, then the update rule of TS-Prod~\ref{alg: ts-prod} is well defined and satisfies $\pi_{t,i}> C_t^2\eta_t^2$ for any arm and loss sequence at all time steps.
\end{lemma}
\begin{proof}
The proof follows by induction. At $t=1$ the statement is true by definition. Let the claim hold at time $t$, then the probability of an arm only decreases when $\tilde\ell_{t,i}-\lambda_{t,i}$ is positive. We look at the cases where $A_t=i$ and $A_t\neq i$ independently.

{\bf Case $A_t=i$:} 
\begin{align*}
    \pi_{t+1,i}&> \pi_{t,i}-2\eta_t\left(\sqrt{\pi_{t,i}}\ell_{t,i}-\eta_t(C_t-\frac{13}{2}\pi_{t,i})\right)\\
    &> (1-13\eta_t^2)\left(\sqrt{\pi_{t,i}}-\frac{\eta_t}{1-13\eta_t^2}\right)^2+(13-\frac{1}{1-13\eta_t^2})\eta_t^2
\end{align*}
This is a quadratic function with minimizer $\frac{\eta_t^2}{(1-13\eta_t^2)^2}<C_t^2\eta_t^2$, hence the value is lower bounded by setting $\pi_{t,i}$ to $C_t^2\eta_t^2$
\begin{align*}
    \pi_{t+1,i}&> C_t^2\eta_t^2(1 -13\eta_t^2)\geq C_t^2\eta_{t+1}^2\geq C_{t+1}^2\eta_{t+1}^2
\end{align*}
{\bf Case $A_t\neq i$:} 
\begin{align*}
    \pi_{t+1,i}&= \pi_{t,i}-2\eta_t\pi_{t,i}^\frac{3}{2}\left(\frac{\eta_t(C_t-\frac{13}{2}\pi_{t,A_t})-\ell_{t,A_t}\sqrt{\pi_{t,A_t}}}{\sum_{j=1}^K\pi_{tj}^\frac{3}{2}}\right)\\
    &> \pi_{t,i}-2C_t\eta_t^2\pi_{t,i}^\frac{3}{2}\sqrt{K}
\end{align*}
This is a concave function (in $\pi_{t,i}$) so the minimizer is at either at  $\pi_{t,i}= C_t^2\eta_t^2$ or at $\pi_{t,i}= 1$.
For the latter, we have have $\pi_{t+1,i}>\frac{1}{2}$, so the minimum is obtained for the first case.
\begin{align*}
    \pi_{t+1,i}&>  C_t^2\eta_t^2-2C_t^4\eta_t^5\sqrt{K}>C_t^2\eta_t^2(1 -13\eta_t^2)\geq C_{t+1}^2\eta_{t+1}^2\,.
\end{align*}
\end{proof}

\begin{lemma}[Lemma~\ref{lem: per step TS}]
% \label{lem: per step TS}
For any time $t$ such that $\pi_{t,i}>C_t^2\eta_t^2$, it holds
\begin{align*}
 \ip{\pi_t-u,\ell_t}+\E_t\left[\frac{D_{TS}(u,\pi_{t+1})}{\eta_t}\right]-\frac{D_{TS}(u,\pi_{t})}{\eta_t} \leq \sum_{i=1}^K\left(2\eta_t\sqrt{\pi_{t,i}}(1-\pi_{t,i})-\left(\frac{1}{\eta_t}-\frac{1}{\eta_{t-1}}\right)\frac{u_i - \pi_{t,i}}{\sqrt{\pi_{t,i}}}\right)\,.
\end{align*}
\end{lemma}
\begin{proof}[Proof of Lemma~\ref{lem: per step TS}]
\begin{align*}
&\ip{\pi_t-u,\ell_t}+\E_t\left[\eta_t^{-1}D_{TS}(u,\pi_{t+1})\right]-\eta_t^{-1}D_{TS}(u,\pi_{t})\\
    &=\ip{\pi_t-u,\ell_t} +\E_t\left[\sum_{i=1}^K\frac{u_i-\pi_{t+1,i}}{\eta_t\sqrt{\pi_{t+1,i}}}-\frac{u_i-\pi_{ti}}{\eta_t\sqrt{\pi_{ti}}} +\frac{1}{\eta_t}\left(2\sqrt{\pi_{t+1,i}}-2\sqrt{\pi_{ti}}\right) \right]\\
    &=\ip{\pi_t-u,\ell_t} +\sum_{i=1}^K\left(\frac{u_i}{\eta_t\sqrt{\pi_{ti}}}\E_t\left[\sqrt{\frac{\pi_{ti}}{\pi_{t+1,i}}}-1 \right]
    +
    \frac{\pi_{ti}}{\eta_t\sqrt{\pi_{ti}}}\E_t\left[\sqrt{\frac{\pi_{t+1,i}}{\pi_{ti}}}-1 \right]\right)\\
    &=\ip{\pi_t-u,\ell_t} +\sum_{i=1}^K\Bigg(\frac{u_i}{\eta_t\sqrt{\pi_{ti}}}\E_t\left[\sqrt{1+\frac{\frac{2\eta_t}{\sqrt{\pi_{ti}}}(\tilde\ell_{ti}-\lambda_{ti})}{1-\frac{2\eta_t}{\sqrt{\pi_{ti}}}(\tilde\ell_{ti}-\lambda_{ti})}}-1 \right]\\
    &\qquad\qquad\qquad\qquad\qquad+
    \frac{\pi_{ti}}{\eta_t\sqrt{\pi_{ti}}}\E_t\left[\sqrt{1-\frac{2\eta_t}{\sqrt{\pi_{ti}}}(\tilde\ell_{ti}-\lambda_{ti})}-1 \right]\Bigg)\\
    &\leq\ip{\pi_t-u,\ell_t} +\sum_{i=1}^K\Bigg(\frac{u_i}{\eta_t\sqrt{\pi_{ti}}}\E_t\left[\frac{\eta_t(\tilde\ell_{ti}-\lambda_{ti})}{\sqrt{\pi_{ti}}}+\frac{2\eta_t^2(\tilde\ell_{ti}-\lambda_{ti})^2}{\pi_{ti}(1-\frac{2\eta_t}{\sqrt{\pi_{ti}}}(\tilde\ell_{ti}-\lambda_{ti}))}\right]\\
    &\qquad\qquad\qquad\qquad\qquad+
    \frac{\pi_{ti}}{\eta_t\sqrt{\pi_{ti}}}\E_t\left[-\frac{\eta_t}{\sqrt{\pi_{ti}}}(\tilde\ell_{ti}-\lambda_{ti})\right]\Bigg)\\
    &\leq\ip{\pi_t-u,\ell_t} +\sum_{i=1}^K\Bigg(\frac{u_i-\pi_{ti}}{\pi_{ti}}\E_t\left[\tilde\ell_{ti}-\lambda_{ti}\right]+
    \frac{4\eta_t u_{ti}}{\pi_{ti}^\frac{3}{2}}\E_t\left[(\tilde\ell_{ti}-\lambda_{ti})^2\right]\Bigg) \quad\text{Setting $C_t\geq 4$}\\
    &\leq\sum_{i=1}^K (\pi_{ti} - u_i)\ell_{ti} + \frac{u_i - \pi_{ti}}{\pi_{ti}}(\pi_{ti}(\ell_{ti} - c_t) - \eta_t C_t\sqrt{\pi_{ti}}(1-\pi_{ti}))
    + \frac{13}{2}\frac{\eta_t u_i}{\sqrt{\pi_{ti}}}(1-\pi_{ti}) \tag{Lemma~\ref{lem:ts_moments}}\\
    &\leq  \sum_{i=1}^K\Bigg(\left(\frac{\pi_{ti}-u_i}{\sqrt{\pi_{ti}}}\right)\eta_t (C_t-\frac{13}{2}\pi_{ti})+
    \frac{13\eta_t u_{i}}{2\sqrt{\pi_{ti}}}(1-\pi_{ti})\Bigg) \\
    &= \sum_{i=1}^K\left(\frac{13}{2}\eta_t\sqrt{\pi_{ti}}(1-\pi_{ti})-(C_t-\frac{13}{2})\eta_t\frac{u_i-\pi_{ti}}{\sqrt{\pi_{ti}}}\right)\,.
\end{align*}
\end{proof}
% \begin{lemma}
% If $\pi_{ti}\leq 36\eta_t^2$, then 
% \begin{align*}
%     (1-\pi_{ti})-\frac{\sqrt{\pi_{ti}}\sum_{j=1}^K(\pi_{tj}-\pi_{tj}^2)}{\sum_{j=1}^K\pi_{tj}^\frac{3}{2}}\geq 0
% \end{align*}
% \end{lemma}
% \begin{proof}
% Recall $\eta_t\leq \frac{1}{12\sqrt{K}}$.
% \begin{align*}
%     (1-\pi_{ti})-\frac{\sqrt{\pi_{ti}}\sum_{j=1}^K(\pi_{tj}-\pi_{tj}^2)}{\sum_{j=1}^K\pi_{tj}^\frac{3}{2}}\geq (1-36\eta_{t}^2)-\frac{6\eta_t}{\sum_{j=1}^K\pi_{tj}^\frac{3}{2}}\geq 1-32\eta_t^2-6\eta_t\sqrt{K}\geq 0\,.
% \end{align*}
% \end{proof}
% \begin{lemma}
% If $\pi_{ti}\leq 32\eta_t^2$, then $\pi_{t+1,i} \geq \pi_{ti}(1-\eta_t^2)$
% \end{lemma}

\subsection{Regret analysis}
\begin{lemma}
\begin{align*}
    \E_t\left[\left(\ell_{ti}\mathbb{I}(A_t=i)-\frac{\pi_{ti}\sqrt{\pi_{t,A_t}}\ell_{t,A_t}}{\sum_{j=1}^K\pi_{tj}^\frac{3}{2}}\right)^2\right]\leq 5\pi_{ti}(1-\pi_{ti})\,.
\end{align*}
\end{lemma}
\begin{proof}
\begin{align*}
    \E_t\left[\left(\ell_{ti}\mathbb{I}(A_t=i)-\frac{\pi_{ti}\sqrt{\pi_{t,A_t}\ell_{t,A_t}}}{\sum_{j=1}^K\pi_{tj}^\frac{3}{2}}\right)^2\right]&=\pi_{ti}\left(\frac{\sum_{j\neq i}\pi_{tj}^\frac{3}{2}}{\sum_{j=1 }^K\pi_{tj}^\frac{3}{2}}\right)^2 \ell_{ti}^2+ \pi_{ti}\frac{\pi_{ti}\sum_{j\neq i}\pi_{tj}^2\ell_{tj}^2}{\left(\sum_{j=1 }^K\pi_{tj}^\frac{3}{2}\right)^2}\\
    &\leq 2\pi_{ti}\,.
\end{align*}
If $\pi_{ti}\leq \frac{3}{5}$, the claim is proven, it remains to look at the case $\pi_{ti}>\frac{3}{5}$
\begin{align*}
    \left(\frac{\sum_{j\neq i}\pi_{tj}^\frac{3}{2}}{\sum_{j=1 }^K\pi_{tj}^\frac{3}{2}}\right)^2 + \frac{\pi_{ti}\sum_{j\neq i}\pi_{tj}^2}{\left(\sum_{j=1 }^K\pi_{tj}^\frac{3}{2}\right)^2}&\leq \left(\frac{\sum_{j\neq i}\pi_{tj}^\frac{3}{2}}{\sum_{j=1 }^K\pi_{tj}^\frac{3}{2}}\right)^\frac{2}{3} + \frac{\sqrt{\sum_{j\neq i}\pi_{tj}^2}}{\sum_{j=1 }^K\pi_{tj}^\frac{3}{2}} \\
    &\leq \frac{\sum_{j\neq i}\pi_{tj}}{\pi_{ti}}+\frac{\sum_{j\neq i}\pi_{tj}}{\pi_{ti}^\frac{3}{2}}\leq (1-\pi_{ti})\left(\frac{5}{3}+\left(\frac{5}{3}\right)^\frac{3}{2}\right)\\
    &\leq 5(1-\pi_{ti})\,.
\end{align*}
\end{proof}
\begin{lemma}
\begin{align*}
    \left(\sqrt{\pi_{ti}}(1-\pi_{ti})-\frac{\pi_{ti}\sum_{j=1}^K(\pi_{tj}-\pi_{tj}^2)}{\sum_{j=1}^K\pi_{tj}^\frac{3}{2}}\right)^2\leq \pi_{ti}K
\end{align*}
% \jz{$K$ dependency can be improved if necessary}
\end{lemma}
\begin{proof}
Pulling $\pi_{ti}$, the term inside the bracket is upper bounded by 1, hence 
\begin{align*}
    \left(\sqrt{\pi_{ti}}(1-\pi_{ti})-\frac{\pi_{ti}\sum_{j=1}^K(\pi_{tj}-\pi_{tj}^2)}{\sum_{j=1}^K\pi_{tj}^\frac{3}{2}}\right)^2&\leq
    \pi_{ti}\max\left\{1,\left(\frac{\sqrt{\pi_{ti}}}{\sum_{j=1}^K\pi_{tj}^\frac{3}{2}}\right)^2\right\}\\
    &\leq \pi_{ti}K\,.
\end{align*}
\end{proof}

\subsection{Self-bounding trick}
\label{sec:self-bounding}
We now quickly describe how to apply the self-bounding trick from \citet{zimmert2021tsallis}. Assume that we have a regret bound of the form
\begin{align*}
    \sum_{t=1}^T \sum_{i\neq i^*} \pi_{t,i}\Delta_i \leq \sum_{t=1}^T\sum_{i\neq i^*} \frac{a\pi_{t,i} + b\sqrt{\pi_{t,i}}}{\sqrt{t}},
\end{align*}
for some positive $a$ and $b$.
The above inequality implies
\begin{align*}
    \frac{1}{3}\sum_{t=1}^T \sum_{i\neq i^*} \pi_{t,i}\Delta_i \leq \sum_{t=1}^T\sum_{i\neq i^*} \pi_{t,i}\left(\frac{a}{\sqrt{t}} - \frac{\Delta_i}{3}\right) + \sqrt{\pi_{t,i}}\left(\frac{b}{\sqrt{t}} - \frac{\Delta_i\sqrt{\pi_{t,i}}}{3}\right).
\end{align*}
For a fixed $i$ the term $\left(\frac{a}{\sqrt{t}} - \frac{\Delta_i}{3}\right) \leq 0$ if $t \geq \frac{9a^2}{\Delta_i^2}$ and so the maximum regret from
\begin{align*}
    \sum_{t=1}^T \pi_{t,i}\left(\frac{a}{\sqrt{t}} - \frac{\Delta_i}{3}\right) \leq \sum_{t=1}^{\lfloor\frac{9a^2}{\Delta_i^2}\rfloor} \frac{a}{\sqrt{t}} \leq \frac{6a^2}{\Delta_i}.
\end{align*}
Further the term $\sqrt{\pi_{t,i}}\left(\frac{b}{\sqrt{t}} - \frac{\Delta_i\sqrt{\pi_{t,i}}}{3}\right) \leq \frac{2 b^2}{t\Delta_i}$ for $t \geq \frac{4b^2}{\Delta_i^2}$. This implies
\begin{align*}
    \sum_{t=1}^T \sqrt{\pi_{t,i}}\left(\frac{b}{\sqrt{t}} - \frac{\Delta_i\sqrt{\pi_{t,i}}}{3}\right) \leq \sum_{t=1}^{\lfloor \frac{4b^2}{\Delta_i^2}\rfloor} \frac{b}{\sqrt{t}} + \sum_{t=1}^T \frac{2b^2}{\Delta_i t} \leq \frac{8b^2}{\Delta_i} + \frac{2b^2\log(T)}{\Delta_i}.
\end{align*}
Combining the two bounds we have
\begin{align*}
    \sum_{t=1}^T \sum_{i\neq i^*} \pi_{t,i}\Delta_i \leq O\left(\sum_{i\neq i^*}\frac{b^2\log(T)}{\Delta_i} + \frac{a^2}{\Delta_i}\right).
\end{align*}

\section{Proof of Theorem~\ref{thm:tsmd_perturbed}}
We decompose one step of the regret as
\subsection{Bias of the estimator}
We choose $\eta_t = \frac{1}{\sqrt{t}}$ and $\gamma_t = \frac{\sqrt{K}}{t}$. Equation 2 in \citet{kocak2014efficient} implies that 
\begin{align*}
    \sum_{i=1}^K\E[\pi_{t,i}\hat \ell_{t,i}] = \sum_{i=1}^K \E[\pi_{t,i}\ell_{t,i}] - \gamma_t\sum_{i=1}^K \E\left[\frac{\pi_{t,i} \ell_{t,i}}{\pi_{t,i} + \gamma_t}\right] - \frac{1-\xi_t}{\eta_{t+1}}\sum_{i=1}^K \E[\sqrt{\pi_{t+1,i}}].
\end{align*}
and so the total bias contributed from the biased estimators part of the loss, $\frac{\ell_{t,i}\mathbf{1}(I_t=i)}{\pi_{t,i} + \gamma_t}$, is at most $\log(T)K^{3/2}$. The term $\frac{1-\xi_t}{\eta_{t+1}}\sum_{i=1}^K \E[\sqrt{\pi_{t+1,i}}]$ in the regret bound will end up contributing $O(\frac{1}{\sqrt{t}}\sum_{i\neq i^*}\E[\sqrt{\pi_{t,i}}])$ because the choice of $\xi_t$ and $\eta_t$ imply $(1-\xi_t)/\eta_{t+1} = O(\frac{1}{\sqrt{t}})$ and $\frac{1-\xi_t}{\eta_{t+1}}\sqrt{\pi_{t,i^*}}$ is canceled out by $ - \langle u, \frac{1-\xi_t}{\eta_{t+1}}\frac{1}{\sqrt{\pi_{t}}}\rangle = -\frac{1-\xi_t}{\eta_{t+1}}\frac{1}{\sqrt{\pi_{t,i^*}}}$.
\subsection{Proof of Lemma~\ref{lem:perturbation_bound}}
We work under the following assumption which is satisfied with the choice of $\eta_t$ and $\gamma_t$.
\begin{assumption}
\label{assm:bounded}
Assume that for all $t \in [T], i\in [K]$ it holds that $|\eta_t\sqrt{\pi_{t,i}\hat L_{t,i}}| \leq \frac{1}{4}$.
\end{assumption}
\begin{proof}[Lemma~\ref{lem:perturbation_bound}]
Let $\tilde \ell = \ell + \epsilon$, we consider
\begin{align}
\label{eq:ts_lin}
    \frac{1}{(1 + \eta \sqrt{\pi}(\ell+\epsilon))^2} = 1 - 2\eta\sqrt{\pi}\ell.
\end{align}
First we show that under Assumption~\ref{assm:bounded} it there exists an $\epsilon$ s.t. $|\epsilon| \leq O(|\ell|)$ and Equation~\ref{eq:ts_lin} is satisfied. To see this, first consider the case that $\ell \geq 0$. 
For $\epsilon = 0$ under Assumption~\ref{assm:bounded} we have that
\begin{align*}
    \frac{1}{(1 + \eta\sqrt{\pi}\ell)^2} \leq 1 - 2\eta\sqrt{\pi}\ell.
\end{align*}
Further, for $\epsilon = -\ell$, we have that $1 \geq 2\eta\sqrt{\pi}\ell$. Since $\frac{1}{(1 + \eta\sqrt{\pi}(\ell+\epsilon))^2}$ is continuous in $\epsilon$ the Intermediate Value theorem implies the claim. The case where $\ell < 0$ is handled similarly.
This implies $\tilde \ell = \Theta(\ell)$. We now have
\begin{align*}
    \frac{1}{(1 + \eta \sqrt{\pi}(\ell+\epsilon))^2} &= 1 - 2\eta\sqrt{\pi}\ell \iff\\
    \frac{1}{(1 + \eta \sqrt{\pi}\tilde\ell)^2} &= 1 - 2\eta\sqrt{\pi}(\tilde\ell - \epsilon) \iff\\
     \frac{1}{(1 + \eta \sqrt{\pi}\tilde\ell)^2} - 1 + 2\eta\sqrt{\pi}\tilde\ell &= 2\eta\sqrt{\pi}\epsilon \iff\\
      \frac{1 - (1 + \eta\sqrt{\pi}\tilde\ell)^2}{(1 + \eta \sqrt{\pi}\tilde\ell)^2} + 2\eta\sqrt{\pi}\tilde\ell &= 2\eta\sqrt{\pi}\epsilon \iff\\
      -\frac{(2 + \eta\sqrt{\pi}\tilde\ell)}{(1 + \eta \sqrt{\pi}\tilde\ell)^2} + 2\tilde\ell  &= 2\epsilon \iff\\
      \frac{-2 - \eta\sqrt{\pi}\tilde\ell + 2(1 + \eta\sqrt{\pi}\tilde\ell)^2}{(1 + \eta\sqrt{\pi}\tilde\ell)^2} &= 2\epsilon \iff\\
      \frac{3\eta\sqrt{\pi}\tilde\ell + 2\eta^2\pi\tilde\ell^2}{(1 + \eta\sqrt{\pi}\tilde\ell)^2} &= 2\epsilon,
\end{align*}
and so we have $\epsilon = O(\eta\sqrt{\pi}\ell)$.
\end{proof}
We can now bound the first and second moments of $\epsilon_{t,i}$ s.t.
\begin{align*}
    \frac{1}{(1 + \eta_t\sqrt{\pi_{t,i}}(\hat L_{t,i} + \epsilon_{t,i}))^2} = 1 - 2\eta_t\sqrt{\pi_{t,i}}\hat L_{t,i},
\end{align*}
which is sufficient for reducing the IC regret to the 1/2-Tsallis INF regret.
\begin{equation}
\label{eq:perturbation_moments}
\begin{aligned}
    \E[\epsilon_{t,i}] &\leq \eta_t\sqrt{\pi_{t,i}}\\
    \E[\epsilon^2_{t,i}] &\leq \eta_t^2.
\end{aligned}
\end{equation}

\subsection{Stochastic bound}
Let $\dts{t}{u}{w} = \frac{1}{\eta_t}D_{TS}(u,v)$ and let 
\begin{align*}
    \tilde \pi_{t+1,i} = \frac{\pi_{t,i}}{(1+\eta_{t+1}\sqrt{\pi_{t,i}}(\hat \ell_{t,i} + \epsilon_{t,i} - \lambda_t))^2},
\end{align*}
where $\lambda_t = \ell_{t,A_t}$. We note that $\pi_{t+1}$ is now the projection of $\tilde\pi_{t+1}$ onto the simplex. Further by the 3-point rule for Bregman divergence we have that
\begin{align*}
    \langle \hat L_t, \pi_t - u \rangle &= \dts{t}{u}{\pi_t} - \dts{t}{u}{\tilde\pi_{t+1}} + \dts{t}{\pi_t}{\tilde\pi_{t+1}}\\
    &\leq \dts{t}{u}{\pi_t} - \dts{t}{u}{\pi_{t+1}} + \dts{t}{\pi_t}{\tilde\pi_{t+1}}.
\end{align*}
\paragraph{Penalty term}
\begin{lemma}[Lemma~\ref{lem:omd_penalty}]
% \label{lem:omd_penalty}
For stochastic losses the penalty term is bounded as follows
\begin{align*}
    \E[\dts{t+1}{u}{\pi_{t+1}} - \dts{t}{u}{\pi_{t+1}}] &\leq O\Bigg(\frac{\E\left[\left(\sum_{i\neq i^*}\pi_{t+1,i}\right)^2\right]\sqrt{K}\log(t)}{\sqrt{t}}\\
    &+ \frac{1}{\sqrt{t}}\land\frac{\E\left[\left(\sum_{i\neq i^*}\pi_{t+1,i}\right)^2\right]\log(KT)}{\sqrt{t}}\Bigg),
\end{align*}
where $\dts{t}{u}{v} = \frac{1}{\eta_t}D_{TS}(u,v)$ and $\eta_t = \frac{1}{\sqrt{t}}, \gamma_{t} = O(\frac{\sqrt{K}}{t})$.
\end{lemma}
\begin{proof}
for $u = e_{i^*}$:
\begin{align*}
    \dts{t+1}{u}{\pi_{t+1}} - \dts{t}{u}{\pi_{t+1}} &= -\left(\frac{1}{\eta_{t+1}} - \frac{1}{\eta_{t}}\right)\\
    &+ \left(\frac{1}{\eta_{t+1}} - \frac{1}{\eta_{t}}\right)\left(\frac{1}{\sqrt{\pi_{t+1,i^*}}} - 1\right)\\
    &+\left(\frac{1}{\eta_{t+1}} - \frac{1}{\eta_{t}}\right)\sum_{i=1}^K \sqrt{\pi_{t+1,i}}\\
    &=\left(\frac{1}{\eta_{t+1}} - \frac{1}{\eta_{t}}\right)\frac{1 - \sqrt{\pi_{t+1,i^*}} + \pi_{t+1,i^*}}{\sqrt{\pi_{t+1, i^*}}}\\
    &+ \left(\frac{1}{\eta_{t+1}} - \frac{1}{\eta_{t}}\right)\sum_{i\neq i^*} \sqrt{\pi_{t+1,i}}\\
    &= \left(\frac{1}{\eta_{t+1}} - \frac{1}{\eta_{t}}\right)\frac{(1 - \sqrt{\pi_{t+1,i^*}})^2}{\sqrt{\pi_{t+1,i^*}}}\\
    &+\left(\frac{1}{\eta_{t+1}} - \frac{1}{\eta_{t}}\right)\sum_{i\neq i^*} \sqrt{\pi_{t+1,i}}.
\end{align*}
From the update in Equation~\ref{eq:tsmd_perturbed} we have
\begin{align*}
    \frac{1}{\sqrt{\pi_{t+1,i^*}}} = \sqrt{K} + \sum_{s=1}^t \eta_s(\hat L_{s,i^*} + \epsilon_{s, i^*}),
\end{align*}
which implies
\begin{align*}
    \frac{(1 - \sqrt{\pi_{t+1,i^*}})^2}{\sqrt{\pi_{t+1,i^*}}} \leq \sqrt{K}(1-\sqrt{\pi_{t+1,i^*}})^2
    +(1-\sqrt{\pi_{t+1,i^*}})^2\sum_{s=1}^t \eta_s(\hat L_{s,i^*} + \epsilon_{s, i^*}).
\end{align*}
First we bound $(1-\sqrt{\pi_{t+1,i^*}})^2$:
\begin{align*}
    (1-\sqrt{\pi_{t+1,i^*}})^2 &= \left(1 - \sqrt{1-\sum_{i\neq i^*} \pi_{t+1,i}}\right)^2 = \left(\frac{\sum_{i\neq i^*} \pi_{t+1, i}}{1 + \sqrt{1-\sum_{i\neq i^*} \pi_{t+1,i}}}\right)^2\\
    &\leq \left(\sum_{i\neq i^*} \pi_{t+1, i}\right)^2.
\end{align*}
In the stochastic case WLOG we can take $\E[\ell_{t,i^*}] = 0, \forall t\in[T]$.
We have
\begin{align*}
    \hat L_{t,i^*} &= \hat \ell_{t,i^*} - \frac{1 - \xi_t}{\eta_{t+1}\sqrt{\pi_{t,i^*}}} - \frac{1}{\sum_{i'} \pi_{t,i'}^{3/2}}\left(\sum_{i'}\sqrt{\pi_{t,i'}}\hat\ell_{t,i'} - \frac{1 - \xi_t}{\eta_{t+1}}\right)\\
    &\leq \hat \ell_{t,i^*} + \left(\frac{1}{\eta_{t+1}} - \frac{1}{\eta_{t}}\right)\frac{1}{\sum_{i'} \pi_{t,i'}^{3/2}} - \frac{1 - \xi_t}{\eta_{t+1}\sqrt{\pi_{t,i^*}}}.
\end{align*}
Next we are going to bound $\eta_s \epsilon_{s,i^*}$ using a concentration argument. We have
\begin{align*}
    \sum_{s=1}^t \eta_s\E[\epsilon_{s, i^*}] \leq \sum_{s=1}^t \eta_s^2.
\end{align*}
the bound on $\epsilon_{t,i}$.
Further, we need to use a concentration inequality for $\sum_{s=1}^t \eta_s \epsilon_{s,i}$ since we can not push the expectation through $(1-\sqrt{\pi_{t+1, i^*}})^2$. Recall that $\epsilon_{s,i} = O(\eta_s\sqrt{\pi_{s,i}}\hat L_{s,i}) = O(1)$. Freedman's inequality now implies that
\begin{align*}
    \sum_{s=1}^t \eta_s \epsilon_{s,i} \leq O\left(\sum_{s=1}^t\eta_s^2 + \sqrt{\sum_{s=1}^t\eta_s^4\log(TK)} + \log(TK)\right) = O(\log(KT)),
\end{align*}
w.p. $1 - 1/(KT)$ for any fixed $i\in[K],t\in[T]$. Finally, the remaining term of $\hat L_{t,i^*}$ is $\hat \ell_{t,i^*}$ which is bounded in expectation by $0$ from our assumption that $\E[\ell_{t,i^*}] = 0$.
We have
\begin{align*}
    &\E\left[\left(\frac{1}{\eta_{t+1}} - \frac{1}{\eta_{t}}\right)\frac{(1 - \sqrt{\pi_{t+1,i^*}})^2}{\sqrt{\pi_{t+1,i^*}}}\right]\\
    \leq & \sqrt{K}\left(\frac{1}{\eta_{t+1}} - \frac{1}{\eta_{t}}\right)\E\left[\left(\sum_{i\neq i^*}\pi_{t+1,i}\right)^2\right] + \left(\frac{1}{\eta_{t+1}} - \frac{1}{\eta_{t}}\right)\E\left[\left(\sum_{i\neq i^*}\pi_{t+1,i}\right)^2\sum_{s=1}^t \eta_s(\hat L_{s,i^*} + \epsilon_{s, i^*})\right]\\
    \leq & \sqrt{K}\left(\frac{1}{\eta_{t+1}} - \frac{1}{\eta_{t}}\right)\E\left[\left(\sum_{i\neq i^*}\pi_{t+1,i}\right)^2\right] + \sum_{s=1}^t \frac{\eta_s}{\sqrt{t}}\E[\hat \ell_{s,i^*}]\\
    &+ \left(\frac{1}{\eta_{t+1}} - \frac{1}{\eta_{t+1}}\right)\E\left[\left(\sum_{i\neq i^*}\pi_{t+1,i}\right)^2\sum_{s=1}^t\frac{1}{s}\frac{1}{\sum_{i'} \pi_{s,i'}^{3/2}}\right]\\
    &+ O\left(\E\left[\left(\sum_{i\neq i^*}\pi_{t+1,i}\right)^2\right]\frac{\log(KT)}{\sqrt{t}}\right)\\
    \leq &O\left(\frac{\E\left[\left(\sum_{i\neq i^*}\pi_{t+1,i}\right)^2\right](\sqrt{K}\log(t) + \log(KT))}{\sqrt{t}}\right).
    % \\+ &O\left(\frac{K^{3/2}\E\left[\left(\sum_{i\neq i^*}\pi_{t+1,i}\right)^2\right]}{t}\right).
\end{align*}
Further, if we can directly bound $\E\left[\left(\sum_{i\neq i^*}\pi_{t+1,i}\right)^2\sum_{s=1}^t \eta_s(\hat L_{s,i^*} + \epsilon_{s, i^*})\right] \leq \sum_{s=1}^t \E\left[\eta_s(\hat L_{s,i^*} + \epsilon_{s, i^*})\right]$ which would imply the bound
\begin{align*}
    \E\left[\left(\frac{1}{\eta_{t+1}} - \frac{1}{\eta_{t}}\right)\frac{(1 - \sqrt{\pi_{t+1,i^*}})^2}{\sqrt{\pi_{t+1,i^*}}}\right] \leq O\left(\frac{\E\left[\left(\sum_{i\neq i^*}\pi_{t+1,i}\right)^2\right]\sqrt{K}\log(t)}{\sqrt{t}} + \frac{1}{\sqrt{t}}\right).
\end{align*}
\end{proof}

\paragraph{Stability term.}
Recall that the stability term is $\dts{t}{\pi_t}{\tilde\pi_{t+1}}$. This term is bounded in a standard way. We proceed to do so as follows for any $t \geq 4\sqrt{K}$:
\begin{lemma}[Lemma~\ref{lem:omd_stability}]
% \label{lem:omd_penalty}
For stochastic losses the stability term is bounded as follows
\begin{align*}
    \E[\dts{t}{\pi_t}{\tilde\pi_{t+1}}] \leq O\left(\frac{1}{\sqrt{t}}\sum_{i=1}^K \sqrt{\pi_{t,i}}(1-\pi_{t,i}) + \frac{K\sqrt{\pi_{t,i}}}{t^2} + \frac{K}{t}\right),
\end{align*}
where $\dts{t}{u}{v} = \frac{1}{\eta_t}D_{TS}(u,v)$.
\end{lemma}
\begin{proof}
We have the following
\begin{align*}
    \dts{t}{\pi_t}{\tilde\pi_{t+1}} &= \frac{1}{\eta_{t}}\sum_{i=1}^K 2\sqrt{\pi_{t+1,i}} - 2\sqrt{\pi_{t,i}} - \frac{1}{\sqrt{\pi_{t+1,i}}}(\pi_{t+1,i} - \pi_{t,i})\\
    &=\frac{1}{\eta_{t}}\sum_{i=1}^K \sqrt{\pi_{t+1,i}} - 2\sqrt{\pi_{t,i}} + \sqrt{\pi_{t,i}}\left(1 + \eta_{t+1}\sqrt{\pi_{t,i}}(\hat \ell_{t,i} + \epsilon_{t,i} - \lambda_t)\right)\\
    &= \frac{1}{\eta_t}\sum_{i=1}^K \sqrt{\pi_{t+1,i}} - \sqrt{\pi_{t,i}} + \eta_{t+1}\pi_{t,i}(\hat \ell_{t,i} + \epsilon_{t,i} - \lambda_t)\\
    &=\frac{1}{\eta_t}\sum_{i=1}^K \sqrt{\pi_{t,i}}\left(\frac{1}{1 + \eta_{t+1}\sqrt{\pi_{t,i}}(\hat \ell_{t,i} + \epsilon_{t,i} - \lambda_t)} - 1\right) + \eta_{t+1}\pi_{t,i}(\hat \ell_{t,i} + \epsilon_{t,i} - \lambda_t)\\
    &\leq \frac{1}{\eta_t}\sum_{i=1}^K \sqrt{\pi_{t,i}}\left(-\eta_{t+1}\sqrt{\pi_{t,i}}(\hat \ell_{t,i} + \epsilon_{t,i} - \lambda_t) + 2\eta_{t+1}^2\pi_{t,i}(\hat \ell_{t,i} + \epsilon_{t,i} - \lambda_t)^2\right)\\
    &+\eta_{t+1}\pi_{t,i}(\hat \ell_{t,i} + \epsilon_{t,i} - \lambda_t) \tag{$\frac{1}{1+x} \leq 1 - x + 2x^2$ for $x \geq -\frac{1}{2}$}\\
    &\leq 2\eta_t\sum_{i=1}^K\pi_{t,i}^{3/2}(\hat \ell_{t,i} + \epsilon_{t,i} - \lambda_t)^2,
\end{align*}
where for the second to last inequality we only need to check $\eta_{t+1}\sqrt{\pi_{t,i}}(\epsilon_{t,i} - \lambda_t) \geq -\frac{1}{2}$. We have $\eta_{t+1}\sqrt{\pi_{t,i}}\lambda_t \geq - \frac{1}{\sqrt{t}}$ and Lemma~\ref{lem:perturbation_bound} implies
\begin{align*}
    \eta_{t+1}\sqrt{\pi_{t,i}}\epsilon_{t,i} \geq -\eta_t^2\pi_{t,i}|\hat L_{t,i}| \geq -\frac{\sqrt{2K}}{t}.
\end{align*}
We bound $\E[\pi_{t,i}^{3/2}(\hat \ell_{t,i} + \epsilon_{t,i} - \lambda_t)^2] \leq 2\E[\pi_{t,i}^{3/2}\epsilon_{t,i}^2] + 2\E[\pi_{t,i}^{3/2}(\hat \ell_{t,i} - \lambda_t)^2]$. For the first term we have
\begin{align*}
    2\E[\eta_t^2\pi_{t,i}^{5/2}|\hat L_{t,i}|^2] \leq O\left(\frac{K\pi_{t,i}}{t}\right).
\end{align*}
For the second term the bound proceeds as in \citet{zimmert2021tsallis}
\begin{align*}
    \E\left[\pi_{t,i}^{3/2}\left(\frac{\ell_{t,i}\mathbb{I}(i=A_t)}{\pi_{t,i} + \gamma_t} - \ell_{t,A_t}\right)^2\right] &= \sum_{j\neq i} \E[\pi_{t,j}\pi_{t,i}^{3/2}\ell_{t,j}^2] + \E\left[\pi_{t,i}^{5/2}\left(\frac{\ell_{t,i}}{\pi_{t,i} + \gamma_t} - \ell_{t,i}\right)^2\right]\\
    &\leq \E[\pi_{t,i}^{3/2}(1-\pi_{t,i})] + \E\left[\frac{\pi_{t,i}^{5/2}}{(\pi_{t,i} + \gamma_t)^2}(1 - \pi_{t,i} - \gamma_t)^2\right]\\
    &\leq 3\E[\sqrt{\pi_{t,i}}(1-\pi_{t,i})] + 2\gamma_t^2\E[\sqrt{\pi_{t,i}}].
\end{align*}
\end{proof}
\paragraph{Self-bounding the regret for stochastic losses.}
\begin{proof}[Theorem~\ref{thm:tsmd_perturbed}, stochastic losses]
Combining the bound in Lemma~\ref{lem:omd_penalty} and Lemma~\ref{lem:omd_stability} we have that the total regret is bounded as follows
\begin{align*}
    &O\left(\sum_{i\neq i^*}\sum_{t=T_{0,i}}^T\sqrt{\pi_{t,i}}\left(\frac{1}{\sqrt{t}} - \sqrt{\pi_{t,i}}\left(\Delta_i - \frac{\log(KT)}{\sqrt{t}}\right)\right) + \sum_{t=1}^{T_{0,i}-1}\frac{1}{\sqrt{t}}\right)\\
    + &O\left(\sum_{t=T_{0}}^T \sum_{i\neq i^*}\sqrt{\pi_{t,i}}\left(\frac{1}{\sqrt{t}} - \sqrt{\pi_{t,i}}\left(\Delta_i - \frac{\sqrt{K}\log(t)}{\sqrt{t}}\right)\right) + \sum_{t=1}^{T_{0}-1} \frac{\pi_{t,i}\sqrt{K}\log(t)}{\sqrt{t}}\right)\\
    + &O\left(K^{3/2}\right)
\end{align*}
In the above we bound the lower order term from the stability as $\sum_{t=1}^T\sum_{i=1}^K \gamma_t^2\sqrt{\pi_{t,i}} = O(K^{3/2})$ and decompose the regret into four parts. The first and second line correspond to the two terms from the penalty bound. Each of the two lines are decomposed into two terms. The first term is the result of the self-bounding trick and the second term is the additional regret for the initial number of rounds before the self-bounding trick can be applied.

We repeatedly use the following inequality $2a\sqrt{x} - bx \leq \frac{a^2}{b}$, which holds for $a,b \geq 0$. We focus on the first line of the decomposition. For $T_{0,i} = \frac{\log^2(KT)}{4\Delta_i^2}$ we have $\sqrt{\pi_{t,i}}\left(\frac{1}{\sqrt{t}} - \sqrt{\pi_{t,i}}\left(\Delta_i - \frac{\log(KT)}{\sqrt{t}}\right)\right) \leq \frac{2}{t\Delta_i}$ and $\sum_{t=1}^{T_{0,i}-1}\frac{1}{\sqrt{t}} = O(\frac{\log(KT)}{\Delta_i})$. For the second line of the decomposition we take $T_0 = 8\frac{K\log^2(K/\Delta_{min})}{\Delta_{min}}$, where $\Delta_{min}$ is the smallest non-zero expected loss. We note that 
\begin{align*}
    \frac{\sqrt{K}\log(T_0)}{\sqrt{T_0}} = \Delta_{min}\frac{\log(8K\log^2(1/\Delta_{min}))}{8\log(K/\Delta_{min})} = \Delta_{min}\left(\frac{\log(K)}{8\log(K/\Delta)} + \frac{\log(16\log(1/\Delta_{\min})}{8\log(K/\Delta_{min})}\right) \leq \frac{\Delta_{min}}{2},
\end{align*}
for any $\Delta_{min} \leq \frac{1}{2024}$. If $\Delta_{min} > \frac{1}{2024}$, we take $T_0 = 8\frac{K\log^2(2024)}{\Delta_{min}}$. The above implies that $\sqrt{\pi_{t,i}}\left(\frac{1}{\sqrt{t}} - \sqrt{\pi_{t,i}}\left(\Delta_i - \frac{\sqrt{K}\log(t)}{\sqrt{t}}\right)\right) \leq \frac{2}{t\Delta_i}$ and further $\sum_{i\neq i^*}\sum_{t=1}^{T_0 - 1} \frac{\pi_{t,i}\sqrt{K}\log(t)}{\sqrt{t}} = O(\frac{K\log^2(1/\Delta_{\min})}{\Delta_{\min}})$.
The final regret bound is
\begin{align*}
    O\left(\sum_{i\neq i^*} \frac{\log(T)}{\Delta_i} + \frac{K\log^2(1/\Delta_{\min})}{\Delta_{\min}} + K^{3/2}\right).
\end{align*}
\end{proof}

\subsection{Adversarial losses}
We now present the argument for the regret bound in the adversarial setting.
\begin{proof}[Theorem~\ref{thm:tsmd_perturbed}, adversarial losses]
We recall the update for Algorithm 2 in \citet{fang2022online}
\begin{align*}
    \hat w_{t+1} &= \nabla \Phi(\pi_t) - \eta_t (\hat \ell_t + \epsilon_t),\\
    \hat y_{t+1} &= \xi_t \hat w_{t+1} + (1-\xi_t) \nabla\Phi(\pi_1),\\
    \pi_{t+1} &= \nabla \Phi^*(\hat y_{t+1}),
\end{align*}
where $\Phi$ is the $1/2$-Tsallis potential plus the indicator function for the probability simplex $\Delta^{K-1}$.
% and $\tilde \ell_t = \frac{\ell_{t,i}\mathbf{1}(i_t=i)}{\pi_{t,i} + \gamma_t}$.
Since $\pi_1$ is uniform the second step of the update is equivalent to $\hat y_{t+1} = \xi_t \hat w_{t+1}$. Re-writing the first step of the update we have 
\begin{align*}
    -\hat y_{t+1,i} &= \frac{\xi_t}{\sqrt{\pi_{t,i}}} + \eta_{t+1}(\tilde \ell_{t,i} + \epsilon_{t,i})= \frac{1}{\sqrt{\pi_{t,i}}} + \eta_{t+1}(\tilde\ell_{t,i} + \epsilon_{t,i})- \eta_{t+1}\frac{1-\xi_t}{\eta_{t+1}}\frac{1}{\sqrt{\pi_{t,i}}}\\
    &= \frac{1}{\sqrt{\pi_{t,i}}} + \eta_{t+1}(\hat\ell_{t+1,i} + \epsilon_{t,i}).
\end{align*}
Since $\nabla \Phi^*$ is invariant under constant vector perturbations we finally have
\begin{align*}
    \pi_{t+1,i} &= \nabla \Phi^*(\hat y_{t+1})_i = \nabla \Phi^*\left(-\frac{1}{\sqrt{\pi_{t}}} - \eta_{t+1}(\hat\ell_{t+1} + \epsilon_{t})\right) = \nabla \Phi^*\left(-\frac{1}{\sqrt{\pi_{t}}} - \eta_{t+1}(\hat L_{t+1} + \epsilon_{t})\right)\\
    &=\frac{1}{(1/\sqrt{\pi_{t,i}} + \eta_{t+1}(\hat L_{t,i} + \epsilon_{t,i}))^2}\\
    &= \frac{\pi_{t,i}}{(1 + \eta_{t+1}\sqrt{\pi_{t,i}}(\hat L_{t,i} + \epsilon_{t,i}))^2}.
\end{align*}
And so the update in Algorithm 2 of \citet{fang2022online} is equivalent to the perturbed OMD update which we have shown enjoys an optimistic regret guarantee.
The regret guarantee in the adversarial setting is now recovered from Theorem 3 in \citet{fang2022online}. In particular the theorem guarantees that the regret is bounded as
\begin{align*}
    \sum_{t=1}^T \dts{t}{\pi_{t}}{\nabla F^*(\nabla F(\pi_t) - \eta_t(\hat \ell_t + \epsilon_t))} + \sqrt{KT}.
\end{align*}
Every term in the sum is bounded in the same way as the stability terms, that is $\dts{t}{\pi_{t}}{\nabla F^*(\nabla F(\pi_t) - \eta_t(\hat \ell_t + \epsilon_t))} \leq \sum_{i=1}^K \pi_{t,i}^{3/2}(\hat \ell_{t,i}^2 + \epsilon_{t,i})$. We can now use the bound in the proof of Lemma~\ref{lem:omd_stability} to complete proof of the adversarial bound and the proof of Theorem~\ref{thm:tsmd_perturbed}.
\end{proof}
\fi
\end{document}